\documentclass{article}

\PassOptionsToPackage{numbers, compress}{natbib}
\usepackage[final]{nips_2018}

\usepackage[utf8]{inputenc} % allow utf-8 input
\usepackage[T1]{fontenc}    % use 8-bit T1 fonts
\usepackage{hyperref}       % hyperlinks
\usepackage{url}            % simple URL typesetting
\usepackage{booktabs}       % professional-quality tables
\usepackage{amsfonts}       % blackboard math symbols
\usepackage{nicefrac}       % compact symbols for 1/2, etc.
\usepackage{microtype}      % microtypography
\usepackage{color}
\usepackage{subfigure}
\usepackage{graphicx}
\usepackage{wrapfig}

\usepackage[linesnumbered,algoruled,boxed,lined,algo2e]{algorithm2e}
\usepackage{amsthm}
\usepackage{amsmath}
\usepackage{amssymb}

\newtheorem{assume}{Assumption}
\newtheorem{definition}{Definition}
\newtheorem{theorem}{Theorem}
\newtheorem{lemma}{Lemma}
\newtheorem*{theorem-non}{Theorem}

\DeclareMathOperator{\X}{\mathcal{X}}

\DeclareMathOperator*{\argmin}{arg\,min}

\DeclareMathOperator*{\minimize}{minimize}

\usepackage[title]{appendix}

\title{
Theoretical Linear Convergence of Unfolded ISTA and its Practical Weights and Thresholds
}

\author{
  Xiaohan Chen\thanks{These authors contributed equally and are listed alphabetically.} \\
  Department of Computer Science and Engineering\\
  Texas A\&M University\\
  College Station, TX 77843, USA \\
  \texttt{chernxh@tamu.edu} \\
  \And
  Jialin Liu\footnotemark[1]\\
  Department of Mathematics\\
  University of California, Los Angeles \\
  Los Angeles, CA 90095, USA \\
  \texttt{liujl11@math.ucla.edu} \\
  \AND
  Zhangyang Wang \\
  Department of Computer Science and Engineering\\
  Texas A\&M University\\
  College Station, TX 77843, USA \\
  \texttt{atlaswang@tamu.edu} \\
  \And
  Wotao Yin \\
  Department of Mathematics\\
  University of California, Los Angeles \\
  Los Angeles, CA 90095, USA \\
  \texttt{wotaoyin@math.ucla.edu} \\
}

\begin{document}
% \nipsfinalcopy is no longer used

\maketitle

\begin{abstract}
  In recent years, unfolding iterative algorithms as neural networks has become an
  empirical success in solving sparse recovery problems. However, its
  theoretical understanding is still immature, which prevents us from fully
  utilizing the power of neural networks. In this work, we study unfolded ISTA
  (Iterative Shrinkage Thresholding Algorithm) for sparse signal recovery. We
  introduce a weight structure that is necessary for asymptotic convergence to
  the true sparse signal. With this structure, unfolded ISTA can attain a linear
  convergence, which is better than the sublinear convergence of ISTA/FISTA in
  general cases. Furthermore, we propose to incorporate thresholding in the
  network to perform support selection, which is easy to implement and able to
  boost the convergence rate both theoretically and empirically. Extensive
  simulations, including sparse vector recovery and a compressive sensing
  experiment on real image data, corroborate our theoretical results and
  demonstrate their practical usefulness. We have made our codes publicly
  available\footnote{\url{https://github.com/xchen-tamu/linear-lista-cpss}}.
\end{abstract}

\vspace{-1em}
\section{Introduction}
\vspace{-0.5em}
This paper aims to recover a sparse vector $x^\ast$ from its noisy linear
measurements:
\begin{equation}
\label{eq:linear_model}
  b = A x^\ast + \varepsilon,
\end{equation}
where $ b\in\mathbb{R}^m $, $ x\in\mathbb{R}^n $,
$ A\in\mathbb{R}^{m \times n} $, $ \varepsilon\in\mathbb{R}^m $
is additive Gaussian white noise, and we have $ m \ll n $. (\ref{eq:linear_model}) is an ill-posed, highly under-determined system.
However, it becomes easier to solve if $x^\ast$ is assumed to be sparse, i.e. the cardinality of support of $x^\ast$, $S=\{i|x^\ast_i \neq 0\}$, is small compared to $n$.

A popular approach is to model the problem as the LASSO formulation ($\lambda$ is a scalar):
\begin{equation}
  \minimize_{x} \frac{1}{2}\|b-Ax\|_2^2 + \lambda\|x\|_1
  \label{eq:lasso}
\end{equation}
and solve it using iterative algorithms such as the iterative shrinkage thresholding
algorithm (ISTA) \cite{blumensath2008iterative}:
\begin{equation}
  x^{k+1} = \eta_{\lambda/L}\Big(x^k + \frac{1}{L}A^T(b-Ax^k)\Big),\quad k = 0, 1, 2, \ldots
  \label{eq:ista}
\end{equation}
where $\eta_{\theta}$ is the soft-thresholding function\footnote{Soft-thresholding function is defined in a component-wise way: $\eta_{\theta}(x) = \text{sign}(x)\max(0,|x|-\theta)$} and $L$ is usually taken as the largest eigenvalue of $A^TA$.
In general, ISTA converges sublinearly for any given and fixed
dictionary $A$ and sparse code $x^\ast$~\cite{beck2009fast}.

In \cite{gregor2010learning}, inspired by ISTA, the authors proposed a
learning-based model named Learned ISTA (LISTA).
They view ISTA as a recurrent neural network (RNN) that is illustrated in
Figure~\ref{fig:rnn}, where
$ W_1 = \frac{1}{L}A^T$, $W_2 = I - \frac{1}{L}A^T A$,
$\theta =\frac{1}{L}\lambda $. LISTA, illustrated in Figure~\ref{fig:lista},
unrolls the RNN and truncates it into $K$ iterations:
\begin{equation}
  x^{k+1} = \eta_{\theta^k}(W^k_1 b + W^k_2 x^k), \quad k = 0,1,\cdots,K-1,
  \label{eq:gen_ista}
\end{equation}
leading to a $K$-layer feed-forward neural network with side connections.

Different from ISTA where no parameter is learnable (except the hyper parameter
$\lambda$ to be tuned), LISTA is treated as a specially structured neural
network and trained using stochastic gradient descent (SGD), over a given
training dataset $\{(x^\ast_i,b_i)\}_{i=1}^N$  sampled from some distribution
$\mathcal{P}(x,b)$. All the parameters $\Theta =
\{(W^k_1,W^k_2,\theta^k)\}_{k=0}^{K-1}$ are subject to learning. The training is
modeled as:
\begin{equation}
\label{eq:train}
\minimize_{\Theta} \mathbb{E}_{x^\ast,b}\Big\|x^K\Big( \Theta, b, x^0 \Big) - x^\ast \Big\|_2^2.
\end{equation}

Many empirical results, e.g.,
\cite{gregor2010learning,wang2016learning,wang2016d3,wang2016learningb,wang2016learningc},
show that a trained $K$-layer LISTA (with $K$ usually set to $10 \sim 20$) or
its variants can generalize more than well to unseen samples $(x',b')$ from the
same $\mathcal{P}(x,b)$ and recover $x'$ from $b'$ to the same accuracy within
one or two order-of-magnitude fewer iterations than the original ISTA.
Moreover, the accuracies of the outputs $\{x^k\}$ of the layers $k=1,..,K$
gradually improve.

\begin{figure}[ht]
  \centering
  \begin{tabular}{cc}
  \hspace{-2mm}
\subfigure[RNN structure of ISTA.]{
  \includegraphics[width=0.30\linewidth]{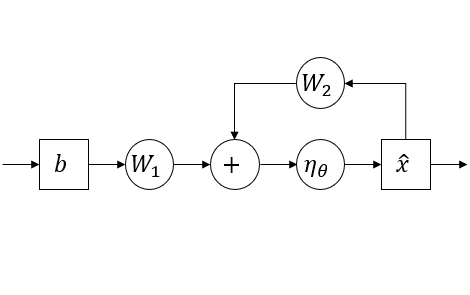}
  \label{fig:rnn}
}
&
\hspace{-4mm}
\subfigure[Unfolded learned ISTA Network.
]{
 	\includegraphics[width=0.68\linewidth]{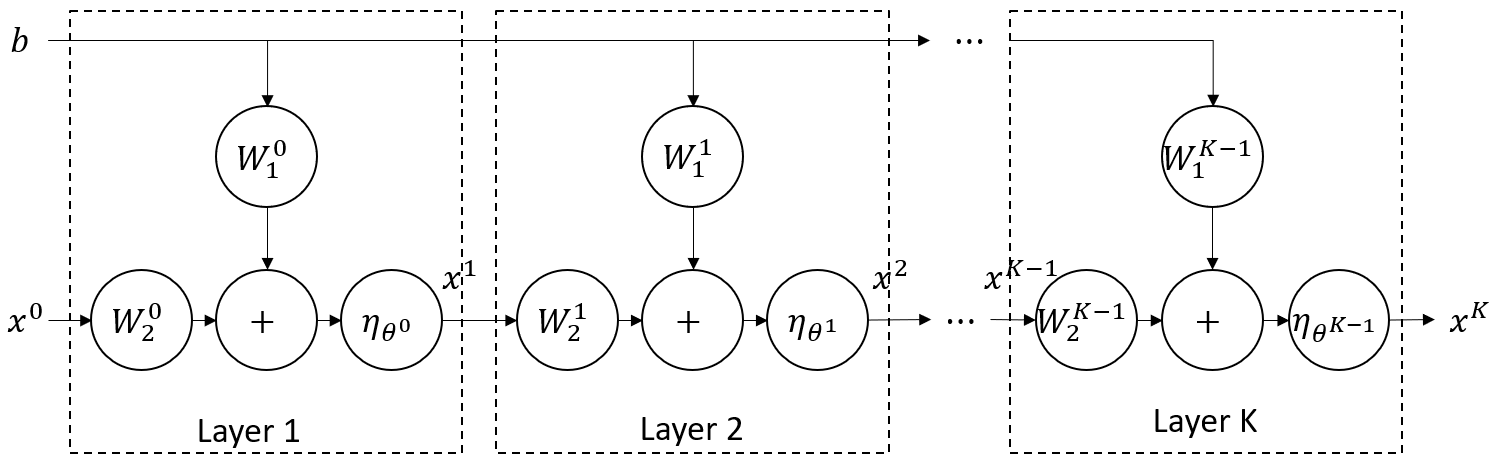}
    \label{fig:lista}
}
\end{tabular}
  \caption{Diagrams of ISTA and LISTA.}
  \vspace{-1em}
  \label{fig:rnn_lista}
\end{figure}

\subsection{Related Works}
\vspace{-0.5em}

Many recent works
\cite{sprechmann2015learning,wang2016sparse,wang2016learning,zhangista,zhou2018sc2net}
followed the idea of \cite{gregor2010learning} to construct feed-forward
networks by unfolding and truncating iterative algorithms, as fast trainable
regressors to approximate the solutions of sparse coding models. On the other
hand, progress has been slow towards understanding the efficient approximation
from a theoretical perspective. The most relevant works are discussed below.

\cite{moreau2017understanding} attempted to explain the mechanism of LISTA by
re-factorizing the Gram matrix of dictionary, which tries to nearly diagonalize
the Gram matrix with a basis that produces a small perturbation of the $\ell_1$
ball. They re-parameterized LISTA into a new factorized architecture that
achieved similar acceleration gain to LISTA. Using an ``indirect'' proof,
\cite{moreau2017understanding} was able to show that LISTA can converge faster
than ISTA, but still sublinearly. Lately, \cite{giryes2018tradeoffs} tried to
relate LISTA to a projected gradient descent descent (PGD) relying on inaccurate
projections, where a trade-off between approximation error and convergence speed
was made possible.

\cite{xin2016maximal} investigated the convergence property of a sibling
architecture to LISTA, proposed in \cite{wang2016learning}, which was obtained
by instead unfolding/truncating the iterative hard thresholding (IHT) algorithm
rather than ISTA. The authors argued that they can use data to train a
transformation of dictionary that can improve its restricted isometry property
(RIP) constant, when the original dictionary is highly correlated, causing IHT
to fail easily. They moreover showed it beneficial to allow the weights to
decouple across layers. However, the analysis in \cite{xin2016maximal} cannot be
straightforwardly extended to ISTA although IHT is linearly convergent
\cite{blumensath2009iterative} under rather strong assumptions.

In \cite{borgerding2017amp}, a similar learning-based model inspired by another
iterative algorithm solve LASSO, approximated message passing (AMP), was
studied. The idea was advanced in \cite{metzler2017learned} to substituting the
AMP proximal operator (soft-thresholding) with a learnable Gaussian denoiser.
The resulting model, called Learned Denoising AMP (L-DAMP), has theoretical
guarantees under the asymptotic assumption named ``state evolution.'' While the
assumption is common in analyzing AMP algorithms, the tool is not directly
applicable to ISTA. Moreover, \cite{borgerding2017amp} shows L-DAMP is MMSE
optimal,  but there is no result on its convergence rate. Besides, we also note
the empirical effort in \cite{borgerding2016onsager} that introduces an Onsager
correction to LISTA to make it resemble AMP.

\vspace{-0.5em}
\subsection{Motivations and Contributions}
\vspace{-0.5em}
We attempt to answer the following questions, which are not fully addressed in
the literature yet:
\begin{itemize}
\vspace{-0.5em}
\itemsep -1pt
\item Rather than training LISTA as a conventional ``black-box'' network, can we
  benefit from exploiting certain dependencies among its parameters
  $\{(W_1^k,W_2^k,\theta^k)\}_{k=0}^{K-1}$ to simplify the network and improve
  the recovery results?
\item Obtained with sufficiently many training samples from the target
  distribution $\mathcal{P}(x,b)$, LISTA works very well. So, we wonder if there
  is a theoretical guarantee to ensure that LISTA (\ref{eq:gen_ista}) converges
  \footnote{The convergence of ISTA/FISTA measures how fast the $k$-th iterate
    proceeds; the convergence of LISTA measures how fast the output of the
    $k$-th layer proceeds as $k$ increases.}
  faster and/or produces a better solution than ISTA (\ref{eq:ista}) when its
  parameters are ideal? If the answer is affirmative, can we quantize the
  amount of acceleration?
\item Can some of the acceleration techniques such as support detection that
  were developed for LASSO also be used to improve LISTA?
\vspace{-0.2em}
\end{itemize}

\textbf{Our Contributions:} %By answering those questions,
this paper aims to introduce more theoretical insights for LISTA and to further
unleash its power. To our best knowledge, this is the first attempt to establish
a theoretical convergence rate (upper bound) of LISTA directly. We also observe
that the \textit{weight structure} and the \textit{thresholds} can speedup the
convergence of LISTA:
\begin{itemize}
\vspace{-0.5em}
\itemsep -1pt
\item We give a result on asymptotic coupling between the weight matrices $W_1^k$ and $W_2^k$. This result leads us to eliminating one of them, thus reducing the number of trainable parameters. This elimination still retains the  theoretical and experimental performance of LISTA.
\item ISTA is generally sublinearly convergent before its iterates settle on a
  support. We prove that, however, there exists a sequence of parameters that
  makes LISTA converge linearly since its first iteration. Our numerical
  experiments support this theoretical result.
\item Furthermore, we introduce a thresholding scheme for \textit{support
  selection}, which is extremely simple to implement and significantly boosts
  the practical convergence. The linear convergence results are extended to
  support detection with an improved rate.
\vspace{-0.5em}
\end{itemize}
Detailed discussions of the above three points will follow after Theorems
\ref{prop:necessary}, \ref{prop:no_ss} and \ref{prop:ss}, respectively. Our
proofs do not rely on any indirect resemblance, e.g., to AMP
\cite{borgerding2016onsager} or PGD \cite{giryes2018tradeoffs}. The theories are
supported by extensive simulation experiments, and substantial performance
improvements are observed when applying the weight coupling and support
selection schemes. We also evaluated LISTA equipped with those proposed
techniques in an image compressive sensing task, obtaining superior performance
over several of the state-of-the-arts.

\vspace{-0.5em}
\section{Algorithm Description}
\label{sec:algo}
\vspace{-0.5em}
We first establish the necessary condition for LISTA convergence, which implies
a partial weight coupling structure for training LISTA. We then describe the
support-selection technique.

\vspace{-0.5em}
\subsection{Necessary Condition for LISTA Convergence and Partial Weight Coupling}
\label{sec:cp}
\vspace{-0.5em}

\begin{assume}[Basic assumptions]
\label{assume:basic}
  The signal $x^\ast$ and the observation noise $\varepsilon$ are sampled from the
  following set:
  \begin{equation}
    \label{eq:x_assume}
    (x^*,\varepsilon) \in \X(B,s,\sigma) \triangleq  \Big\{(x^*,\varepsilon) \Big| |x^\ast_i | \leq B, \forall i, ~\|x^\ast\|_0 \leq s, \|\varepsilon\|_1 \leq \sigma \Big\}.
  \end{equation}
  In other words, $x^\ast$ is bounded and $s$-sparse\footnote{A signal is
  $s$-sparse if it has no more than $s$ non-zero entries.} ($s \geq 2$), and
  $\varepsilon$ is bounded.
\end{assume}

\begin{theorem}[Necessary Condition]
\label{prop:necessary}
  Given $\{W^k_1,W^k_2,\theta^k\}_{k=0}^{\infty}$ and $x^0=0$, let $b$ be
  observed by (\ref{eq:linear_model}) and $\{x^k\}_{k=1}^{\infty}$ be generated
  layer-wise by LISTA~(\ref{eq:gen_ista}). If the following holds uniformly for
  any $(x^*,\varepsilon) \in \X(B,s,0)$ (no observation noise):
  \[x^k\Big( \{W_1^\tau, W_2^\tau, \theta^\tau \}_{\tau=0}^{k-1},b,x^0 \Big) \to x^*,\quad \text{as }k \to \infty\]
  and $\{W^k_2\}_{k=1}^{\infty}$ are bounded
  \[\|W^k_2\|_2\leq B_W,\quad \forall k = 0,1,2,\cdots, \]
  then $\{W^k_1,W^k_2,\theta^k\}_{k=0}^{\infty}$ must satisfy
  \begin{align}
    &W^k_2 -( I - W^k_1A)\to 0, \quad\text{as }k\to\infty \label{eq:couple_way} \\
    &\theta^k\to 0,\quad\text{as }k\to\infty\label{eq:theta_to_0}.
  \end{align}
\end{theorem}
Proofs of the results throughout this paper can be found in the supplementary.
The conclusion (\ref{eq:couple_way}) demonstrates that the weights
$\{W^k_1,W^k_2\}_{k=0}^{\infty}$ in LISTA asymptotically satisfies the following
partial weight coupling structure:
\begin{equation}
  \label{eq:wcp}
  W^k_2 = I - W^k_1A.
\end{equation}
We adopt the above partial weight coupling for all layers, letting $W^k = (W^k_1)^T \in \Re^{m\times n}$, thus simplifying LISTA (\ref{eq:gen_ista}) to:
\begin{equation}
\label{eq:lista_cp}
  x^{k+1}=\eta_{\theta^k}\Big(x^k + (W^k)^\top (b - Ax^k)\Big),\quad k = 0,1,\cdots,K-1,
\end{equation}
where $\{W^k,\theta^k\}_{k=0}^{K-1}$ remain as free parameters to train. Empirical results in Fig. \ref{fig:coupleway} illustrate that the structure (\ref{eq:wcp}), though having fewer parameters, improves the performance of LISTA.

The coupled structure (\ref{eq:wcp}) for soft-thresholding based algorithms was empirically studied in \cite{borgerding2017amp}. The similar structure was also theoretically studied in Proposition 1 of \cite{xin2016maximal} for IHT algorithms using the fixed-point theory, but they let all layers share the same weights, i.e. $W^k_2 = W_2, W^k_1 = W_1,\forall k$.

\vspace{-1em}
\subsection{LISTA with Support Selection}
\vspace{-0.5em}
\label{sec:ss}
We introduce a special thresholding scheme to LISTA, called \textit{support
selection}, which is inspired by ``kicking'' \cite{osher2011fast} in linearized
Bregman iteration. This technique shows advantages on recoverability and
convergence.  Its impact on improving LISTA convergence rate and reducing
recovery errors will be analyzed in Section \ref{sec:convergence}.   With
support selection, at each LISTA layer \textit{before} applying soft
thresholding, we will select a certain percentage of entries with largest
magnitudes, and trust them as ``true support'' and won’t pass them through
thresholding.  Those entries that do not go through thresholding will be
directly fed into next layer, together with other thresholded entires.

Assume we select $p^k\%$ of entries as the trusted support at layer $k$. LISTA
with support selection can be generally formulated as
\vspace{-0.6em}
\begin{equation}
  \label{eq:lista_ss0}
  x^{k+1} = {\eta_\mathrm{ss}}_{\theta^k}^{p^k} \Big(W^k_1 b + W^k_2 x^k\Big), \quad k = 0,1,\cdots,K-1,
\end{equation}
where ${\eta_{ss}}$ is the thresholding operator with support selection,
formally defined as:
\[
({\eta_\mathrm{ss}}_{\theta^k}^{p^k}(v))_i = \left\{
  \begin{array}{lll}
    v_i & : v_i > \theta^k,& i\in S^{p^k}(v), \\
    v_i - \theta^k & : v_i > \theta^k,& i\notin S^{p^k}(v), \\
    0 & : -\theta^k \leq v_i \leq \theta^k &\\
    v_i + \theta^k & : v_i < -\theta^k,& i\notin S^{p^k}(v), \\
    v_i & : v_i < -\theta^k, &i\in S^{p^k}(v),
  \end{array}
\right.
\]
where $S^{p^k}(v)$ includes the elements with the largest $p^k\%$ magnitudes in
vector $v$:
\vspace{-0.6em}
\begin{equation}
\label{eq:spk}
S^{p^k}(v) = \Big\{i_1,i_2,\cdots,i_{p^k}\Big||v_{i_1}| \geq |v_{i_2}| \geq \cdots |v_{i_{p^k}}|\cdots \geq |v_{i_n}|\Big\}.
\end{equation}
To clarify, in (\ref{eq:lista_ss0}), $p^k$ is a hyperparameter to be manually tuned, and $\theta^k$ is a parameter to train. We use an empirical formula to select $p^k$ for layer $k$: $p^k = \min(p\cdot k, p_\mathrm{max})$, where $p$ is a positive constant and $p_\mathrm{max}$ is an upper bound of the percentage of the support cardinality. Here $p$ and $p_\mathrm{max}$ are both hyperparameters to be manually tuned.

If we adopt the partial weight coupling in (\ref{eq:wcp}), then (\ref{eq:lista_ss0}) is modified as
\begin{equation}
\label{eq:lista_ss}
  x^{k+1} = {\eta_\mathrm{ss}}_{\theta^k}^{p^k} \Big(x^k + (W^k)^T (b - Ax^k)\Big),\quad k = 0,1,\cdots,K-1.
\end{equation}

\paragraph{Algorithm abbreviations} For simplicity, hereinafter we will use the
abbreviation ``CP'' for the partial weight coupling in ~(\ref{eq:wcp}), and
``SS'' for the support selection technique. \textit{LISTA-CP} denotes the LISTA
model with weights coupling (\ref{eq:lista_cp}). \textit{LISTA-SS} denotes the
LISTA model with support selection (\ref{eq:lista_ss0}). Similarly,
\textit{LISTA-CPSS} stands for a model using both techniques
(\ref{eq:lista_ss}), which has the best performance. Unless otherwise specified,
\textit{LISTA} refers to the baseline LISTA (\ref{eq:gen_ista}).

\vspace{-1em}
\section{Convergence Analysis}
\label{sec:convergence}
\vspace{-0.5em}

In this section, we formally establish the impacts of (\ref{eq:lista_cp}) and
(\ref{eq:lista_ss}) on LISTA's convergence. The output of the $k^{\text{th}}$
layer $x^k$ depends on the parameters $\{W^\tau,\theta^\tau\}_{\tau=0}^{k-1}$,
the observed measurement $b$ and the initial point $x^0$. Strictly speaking,
$x^k$ should be written as
$x^k\Big( \{W^\tau, \theta^\tau \}_{\tau=0}^{k-1},b, x^0 \Big)$. By the
observation model $b=Ax^*+\varepsilon$, since $A$ is given and $x^0$ can be
taken as $0$, $x^k$ therefore depends on $\{(W^\tau,\theta^\tau)\}_{\tau=0}^k$,
$x^*$ and $\varepsilon$. So, we can write
$x^k\Big( \{W^\tau, \theta^\tau \}_{\tau=0}^{k-1},x^*,\varepsilon \Big)$.
For simplicity, we instead just write $x^k(x^*,\varepsilon)$.
\begin{theorem}[Convergence of LISTA-CP]
  \label{prop:no_ss}
  Given $\{W^k,\theta^k\}_{k=0}^{\infty}$ and $x^0=0$, let $\{x^k\}_{k=1}^{\infty}$ be generated by (\ref{eq:lista_cp}). If Assumption \ref{assume:basic} holds and $s$ is sufficiently small,
  then there exists a sequence of parameters $\{W^k,\theta^k\}$ such that, for all $(x^*,\varepsilon)\in \X(B,s,\sigma)$,
  we have the error bound:
  \begin{equation}
    \label{eq:linear_conv}
    \|x^k(x^*,\varepsilon)-x^\ast\|_2 \leq s B \exp(-ck) + C\sigma,\quad \forall k = 1,2,\cdots,
  \end{equation}
  where $c>0,C>0$ are constants that depend only on $A$ and $s$. Recall $s$
  (sparsity of the signals) and $\sigma$ (noise-level) are defined in
  (\ref{eq:x_assume}).
\end{theorem}
If $\sigma=0$ (noiseless case), (\ref{eq:linear_conv}) reduces to
\vspace{-1mm}
\begin{equation}
\label{eq:linear_noiseless}
\|x^k(x^*,0)-x^\ast\|_2 \leq s B \exp(-ck).
\end{equation}
The recovery error converges to $0$ at a linear rate as the number of layers
goes to infinity. Combined with Theorem \ref{prop:necessary}, we see that the
partial weight coupling structure (\ref{eq:lista_cp}) is both necessary and
sufficient to guarantee convergence in the noiseless case. Fig.
\ref{fig:coupleway} validates (\ref{eq:linear_conv}) and
(\ref{eq:linear_noiseless}) directly.

\textbf{Discussion:} The bound (\ref{eq:linear_noiseless}) also explains why LISTA (or its variants) can converge faster than ISTA and fast ISTA (FISTA) \cite{beck2009fast}. With a proper $\lambda$ (see (\ref{eq:lasso})), ISTA converges at an $O(1/k)$ rate and FISTA converges at an $O(1/k^2)$ rate~\cite{beck2009fast}. With a large enough $\lambda$,
ISTA achieves a linear rate \cite{bredies2008linear,zhang2017new}. With $\bar{x}(\lambda)$ being the solution of LASSO (noiseless case), these results can be summarized as: before the iterates $x^k$ settle on a support\footnote{After $x^k$ settles on a support, i.e.  as $k$ large enough such that $\mathrm{support}(x^{k})$ is fixed, even with small $\lambda$, ISTA reduces to a linear iteration, which has a linear convergence rate \cite{tao2016local}. %Fig.\ref{fig:nmse_ista_lista} also supports this point.
},
\vspace{-1mm}
\[\begin{aligned}
x^k \to \bar{x}(\lambda) \text{ sublinearly},&\quad \|\bar{x}(\lambda) - x^*\| = O(\lambda),\quad \lambda > 0\\
x^k \to \bar{x}(\lambda) \text{ linearly},&\quad \|\bar{x}(\lambda) - x^*\| = O(\lambda),\quad \text{$\lambda$ large enough}.
\end{aligned}\]
Based on the choice of $\lambda$ in LASSO, the above observation reflects an inherent trade-off between convergence rate and approximation accuracy in solving the problem (\ref{eq:linear_model}), see a similar conclusion in \cite{giryes2018tradeoffs}: a larger $\lambda$ leads to faster convergence but a less accurate solution, and vice versa.

However, if $\lambda$ is not constant throughout all iterations/layers, but
instead chosen adaptively for each step, more promising trade-off can
arise\footnote{This point was studied in
\cite{HaleYinZhang2008_sparse,xiao2013proximal} with classical compressive
sensing settings, while our learning settings can learn a good path of
parameters without a complicated thresholding rule  or any manual tuning.}.
LISTA and LISTA-CP, with the thresholds $\{\theta^k\}_{k=0}^{K-1}$ free to
train, actually adopt this idea because $\{\theta^k\}_{k=0}^{K-1}$ corresponds
to a path of LASSO parameters $\{\lambda^k\}_{k=0}^{K-1}$. With extra free
trainable parameters, $\{W^k\}_{k=0}^{K-1}$ (LISTA-CP) or $\{W^k_1,
W^k_2\}_{k=0}^{K-1}$ (LISTA), learning based algorithms are able to converge to
an accurate solution at a fast convergence rate. Theorem \ref{prop:no_ss}
demonstrates the existence of such sequence $\{W^k,\theta^k\}_k$ in LISTA-CP
(\ref{eq:lista_cp}).
The experiment results in Fig. \ref{fig:nmse_ista_lista} show that such
$\{W^k,\theta^k\}_k$ can be obtained by training.
\begin{assume}
  \label{assume:basic2}
  Signal $x^\ast$ and observation noise $\varepsilon$ are sampled from the following set:
  \begin{equation}
    \label{eq:x_assume2}
    (x^*,\varepsilon) \in \bar{\X}(B,\underline{B}, s,\sigma) \triangleq  \Big\{(x^*,\varepsilon) \Big| |x^\ast_i | \leq B, \forall i, ~\|x^\ast \|_1 \geq \underline{B} , \|x^\ast\|_0 \leq s, \|\varepsilon\|_1 \leq \sigma \Big\}.
  \end{equation}
\end{assume}

\begin{theorem}[Convergence of LISTA-CPSS]
  \label{prop:ss}
  Given $\{W^k,\theta^k\}_{k=0}^{\infty}$ and $x^0=0$, let
  $\{x^k\}_{k=1}^{\infty}$ be generated by ~(\ref{eq:lista_ss}).  With the same
  assumption and parameters as in Theorem \ref{prop:no_ss}, the approximation
  error can be bounded for all $(x^*,\varepsilon)\in \X(B,s,\sigma)$:
  \vspace{-2mm}
  \begin{equation}
    \label{eq:linear_ss}
    \|x^k(x^*,\varepsilon) -x^\ast\|_2 \leq s B \exp\Big(-\sum_{t=0}^{k-1}c_{\mathrm{ss}}^t\Big) + C_{\mathrm{ss}}\sigma, \quad \forall k = 1,2,\cdots,
  \end{equation}
  where $c_{\mathrm{ss}}^k\geq c$ for all $k$ and $C_{\mathrm{ss}} \leq C$.

  If Assumption \ref{assume:basic2} holds, $s$ is small enough, and
  $\underline{B} \geq 2C\sigma$ (SNR is not too small), then there exists
  another sequence of parameters $\{\tilde{W}^k,\tilde{\theta}^k\}$ that yields
  the following improved error bound: for all $(x^*,\varepsilon) \in
  \bar{\X}(B,\underline{B}, s,\sigma)$,
  \vspace{-2mm}
  \begin{equation}
    \label{eq:linear_ss2}
    \|x^k(x^*,\varepsilon) -x^\ast\|_2 \leq s B \exp\Big(-\sum_{t=0}^{k-1} \tilde{c}_{\mathrm{ss}}^t\Big) + \tilde{C}_{\mathrm{ss}}\sigma, \quad \forall k = 1,2,\cdots,
  \end{equation}
  where $\tilde{c}_{\mathrm{ss}}^k\geq c$ for all $k$,
  $\tilde{c}_{\mathrm{ss}}^k > c$ for large enough $k$,  and
  $\tilde{C}_{\mathrm{ss}}< C$.
\end{theorem}

The bound in (\ref{eq:linear_ss}) ensures that, with the same assumptions and
parameters, LISTA-CPSS is \textit{at least no worse} than LISTA-CP. The bound in
(\ref{eq:linear_ss2}) shows that, under stronger assumptions, LISTA-CPSS can be
\textit{strictly better} than LISTA-CP in both folds:
$\tilde{c}_{\mathrm{ss}}^k > c$ is the better convergence rate of LISTA-CPSS;
$\tilde{C}_{\mathrm{ss}}< C$ means that the LISTA-CPSS can achieve smaller
approximation error than the minimum error that LISTA can achieve.

\vspace{-1em}
\section{Numerical Results}
\vspace{-0.5em}
For all the models reported in this section, including the baseline LISTA and
LAMP models , we adopt a stage-wise training strategy with learning rate
decaying to stabilize the training and to get better performance, which is
discussed in the supplementary.

\subsection{Simulation Experiments}
\label{sec:simulation}
\vspace{-0.5em}
\textbf{Experiments Setting.}
We choose $m=250, n=500$. We sample the entries of $A$ i.i.d. from the standard Gaussian distribution,
$A_{ij} \sim N(0, 1/m)$ and then normalize its columns to have the unit $\ell_2$ norm. We fix a matrix $A$ in each setting where different networks are compared. To generate sparse vectors
$x^*$, we  decide each of its entry to be non-zero following the Bernoulli distribution with $p_b=0.1$.
The values of  the non-zero entries are sampled from the standard Gaussian distribution. A test set
of 1000 samples generated in the above manner is fixed for all tests in our simulations.

All the networks have $K=16$ layers. In LISTA
models with support selection, we add $p\%$ of
entries into support and maximally select $p_\mathrm{max}\%$ in each layer. We manually tune the value of $p$ and $p_\mathrm{max}$ for the best final performance. With $p_b = 0.1$ and $K=16$, we choose $p=1.2$ for all models in simulation experiments and $p_\mathrm{max}=12$ for LISTA-SS but $p_\mathrm{max}=13$ for LISTA-CPSS. % We need to choose a proper $p$ that is neither too aggressive nor too conservative.
The recovery performance is evaluated by NMSE (in dB):
\[
  \mathrm{NMSE}(\hat{x},x^\ast)= 10 \log_{10}
  \left(\frac{\mathbb{E}\|\hat{x}-x^\ast\|^2}{\mathbb{E}\|x^\ast\|^2}\right),
\]
where $x^\ast$ is the ground truth and $\hat{x}$ is the estimate obtained by the
recovery algorithms (ISTA, FISTA, LISTA, etc.).

\textbf{Validation of Theorem \ref{prop:necessary}.}
In Fig \ref{fig:all_free}, we report two values, $\|W^k_2 - (I - W^k_1A)\|_2$
and $\theta^k$, obtained by the baseline LISTA model (\ref{eq:gen_ista}) trained
under the noiseless setting.  The plot clearly demonstrates that
$W^k_2 \to I - W^k_1A,$ and $\theta^k\to 0$, as $k\to\infty.$ Theorem
\ref{prop:necessary} is directly validated.
\begin{figure}[ht]
  \centering
  \begin{tabular}{cc}
    \hspace{-7mm}
    \subfigure[Weight $W^k_2 \to I-W^k_1A$ as $k\to \infty$.]{
      \includegraphics[width=0.45\linewidth]{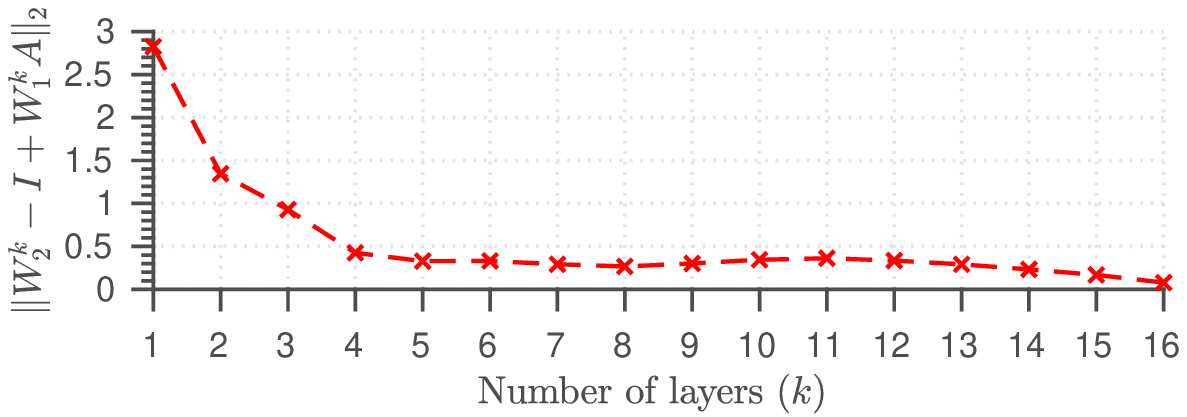}
    }
    &
    \subfigure[The threshold $\theta^k\to 0$.]{
      \includegraphics[width=0.45\linewidth]{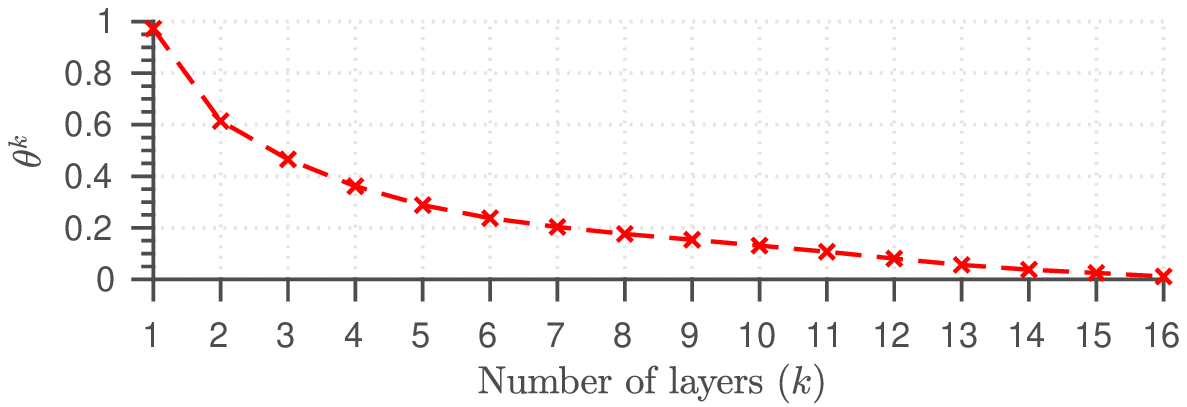}
    }
  \end{tabular}
  \caption{Validation of Theorem \ref{prop:necessary}.}
  \vspace{-1em}
  \label{fig:all_free}
\end{figure}

\textbf{Validation of Theorem \ref{prop:no_ss}.}
We report the test-set NMSE of LISTA-CP (\ref{eq:lista_cp}) in
Fig. \ref{fig:coupleway}. Although (\ref{eq:lista_cp}) fixes the structure
between $W_1^k$ and $W_2^k$, the final performance remains the same with the
baseline LISTA (\ref{eq:gen_ista}), and outperforms AMP, in both noiseless and
noisy cases. Moreover, the output of interior layers in LISTA-CP are even better
than the baseline LISTA. In the noiseless case, NMSE converges exponentially to
$0$; in the noisy case, NMSE converges to a stationary level related with the
noise-level. This supports Theorem \ref{prop:no_ss}: there indeed exist a
sequence of parameters $\{(W^k,\theta^k)\}_{k=0}^{K-1}$ leading to linear
convergence for LISTA-CP, and they can be obtained by data-driven learning.
\begin{figure}[ht]
  \centering
  \begin{tabular}{cc}
    \hspace{-7mm}
    \subfigure[$\ \mathrm{SNR}=\infty$]{
      \includegraphics[width=0.45\linewidth]{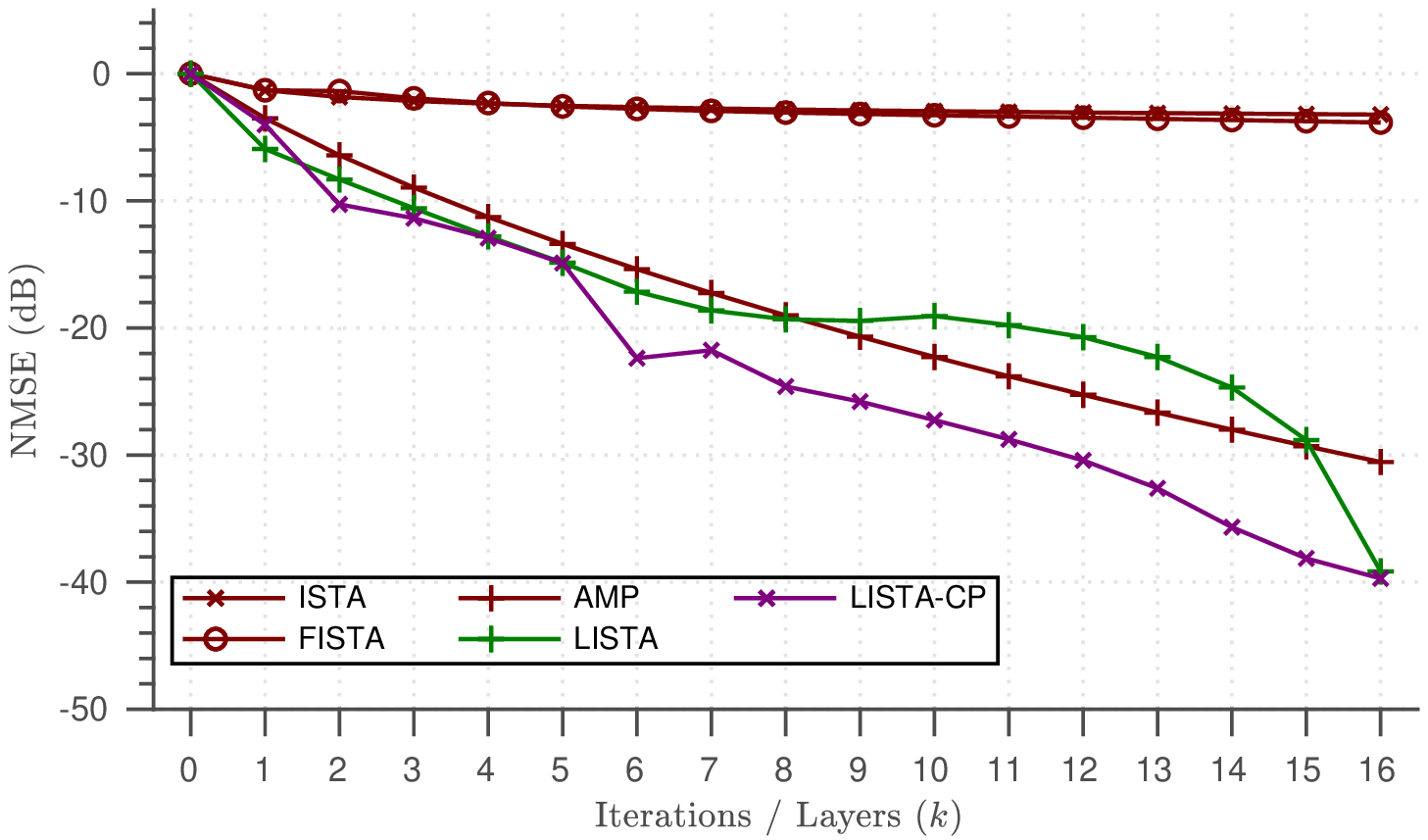}
    }
    &
    \subfigure[$\ \mathrm{SNR}=30$]{
      \includegraphics[width=0.45\linewidth]{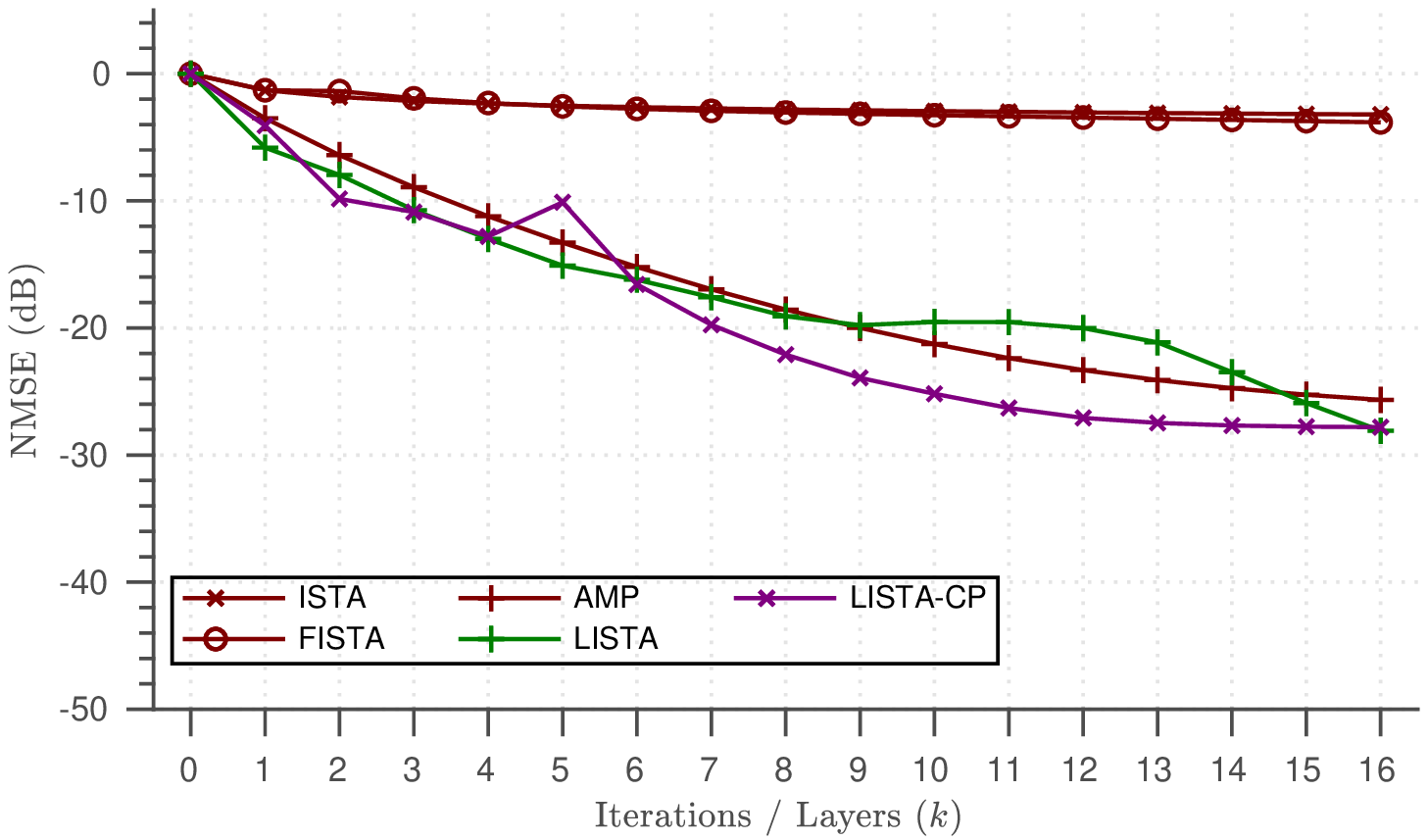}
    }
  \end{tabular}
  \caption{Validation of Theorem \ref{prop:no_ss}.}
  \label{fig:coupleway}
  \vspace{-0.5em}
\end{figure}

\begin{wrapfigure}{rt}{0.65\textwidth}
  \centering
  \includegraphics[width=\linewidth]{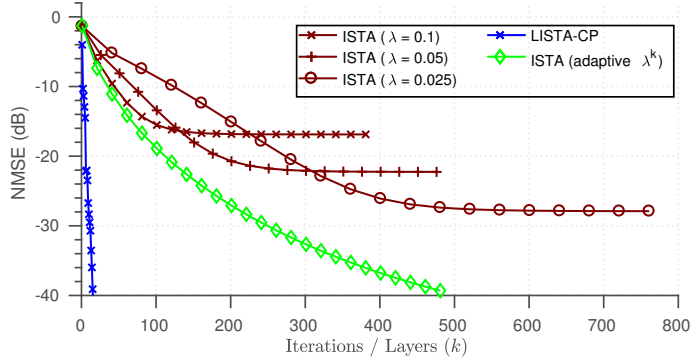}
  \vspace{-1em}
  \caption{Validating Discussion after Theorem \ref{prop:no_ss} (SNR = $\infty$).}
  \vspace{-0.5em}
  \label{fig:nmse_ista_lista}
\end{wrapfigure}

\textbf{Validation of Discussion after Theorem \ref{prop:no_ss}.} In Fig \ref{fig:nmse_ista_lista}, We compare LISTA-CP and ISTA with different $\lambda$s (see the LASSO problem (\ref{eq:lasso})) as well as an adaptive threshold rule similar to one in \cite{HaleYinZhang2008_sparse}, which is described in the supplementary.
As we have discussed after Theorem \ref{prop:no_ss}, LASSO has an inherent tradeoff based on the choice of $\lambda$.  A smaller $\lambda$ leads to a more accurate solution but slower convergence. The adaptive thresholding rule fixes this issue: it uses large $\lambda^k$ for small $k$, and gradually reduces it as $k$ increases to improve the accuracy~\cite{HaleYinZhang2008_sparse}. Except for adaptive thresholds $\{\theta^k\}_k$ ($\theta^k$ corresponds to $\lambda^k$ in LASSO), LISTA-CP has adaptive weights $\{W^k\}_k$, which further greatly accelerate the convergence. Note that we only ran ISTA and FISTA for 16 iterations, just enough and fair to compare them with the learned models. The number of iterations is so small that the difference between ISTA and FISTA is not quite observable.

\textbf{Validation of Theorem \ref{prop:ss}.}
We compare the recovery NMSEs of LISTA-CP (\ref{eq:lista_cp}) and LISTA-CPSS
(\ref{eq:lista_ss}) in Fig. \ref{fig:ss}. The result of the noiseless case (Fig.
\ref{fig:nmse_sinf}) shows that the recovery error of LISTA-SS converges to $0$
at a faster rate than that of LISTA-CP.  The difference is significant with the
number of layers $k \geq 10$, which supports our theoretical result:
``$\tilde{c}_{\text{ss}}^k > c$ as $k$ large enough'' in Theorem \ref{prop:ss}.
The result of the noisy case (Fig. \ref{fig:nmse_s40}) shows that LISTA-CPSS has
better recovery error than LISTA-CP.  This point supports
$\tilde{C}_{\text{ss}}<C$ in Theorem \ref{prop:ss}. Notably, LISTA-CPSS also
outperforms LAMP \cite{borgerding2017amp}, when $k$ > 10 in the noiseless case,
and even earlier as SNR becomes lower.

\begin{figure}[ht]
  \centering
  \begin{tabular}{cc}
    \vspace{-1em}
    \subfigure[Noiseless Case]{
      \includegraphics[width=0.45\linewidth]{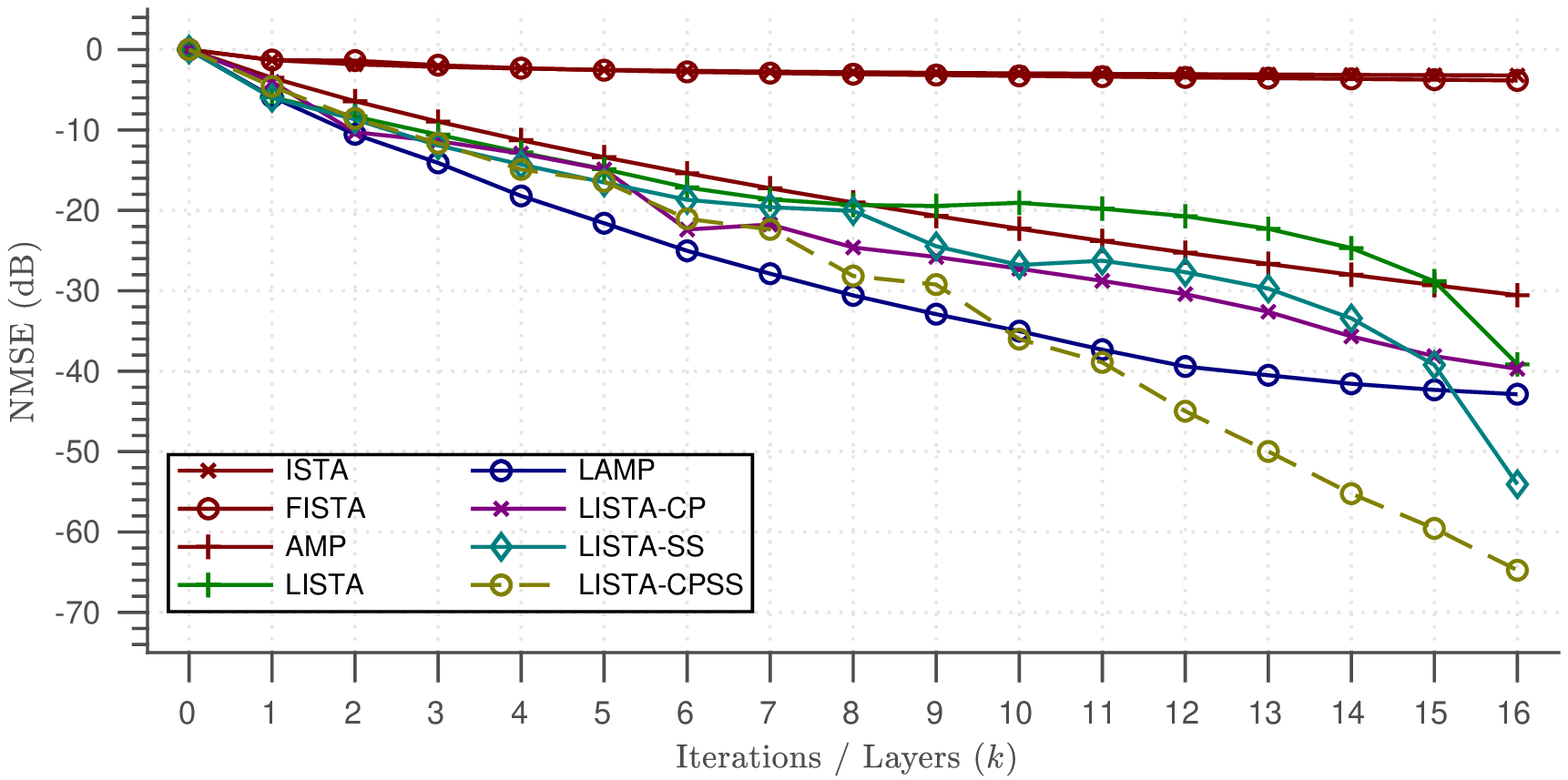}
      \label{fig:nmse_sinf}
    }
    &
    \subfigure[Noisy Case: $\textrm{SNR}$=40dB]{
      \includegraphics[width=0.45\linewidth]{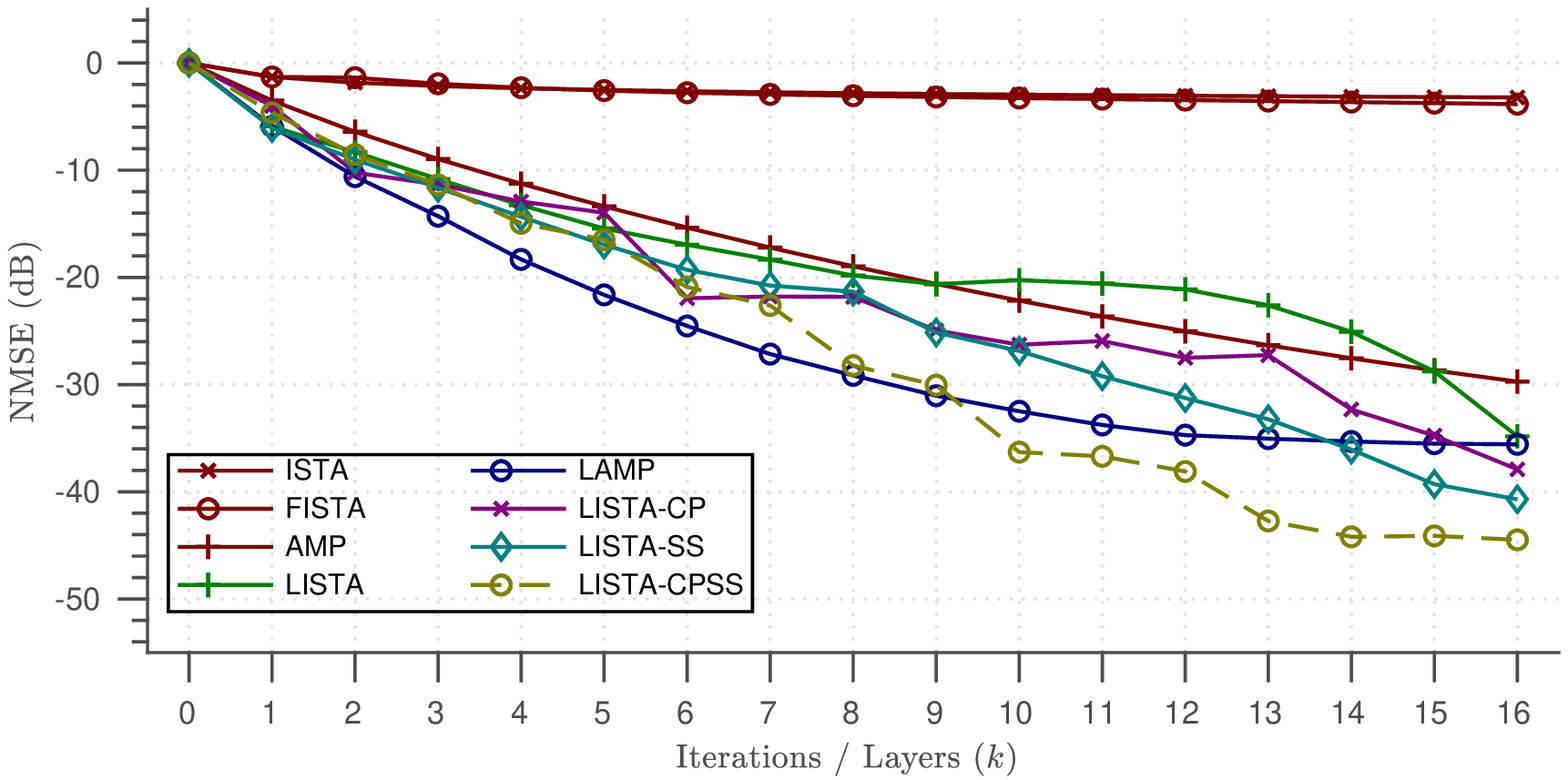}
      \label{fig:nmse_s40}
    } \\
    \subfigure[Noisy Case: $\textrm{SNR}$=30dB]{
      \includegraphics[width=0.45\linewidth]{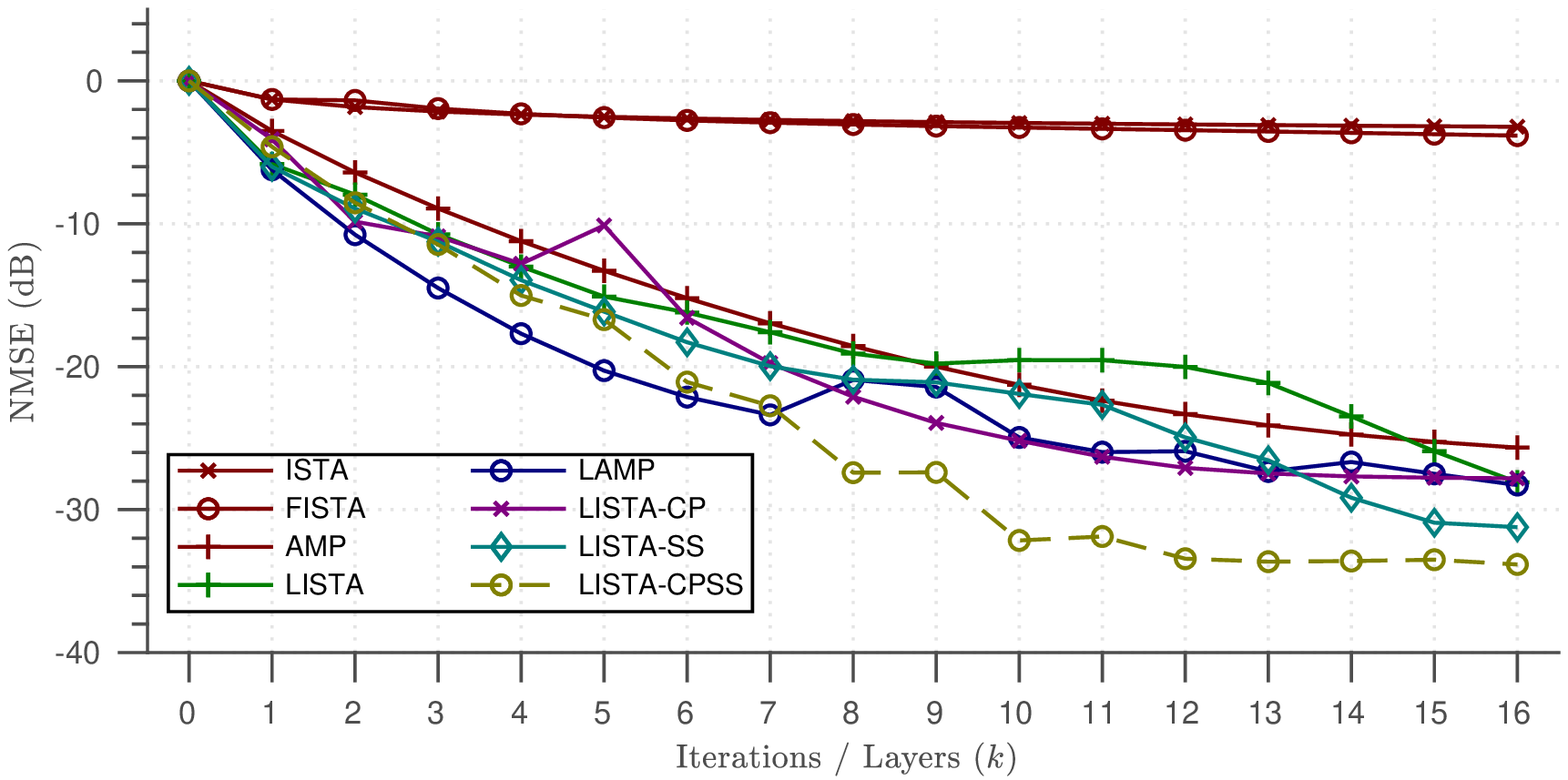}
      \label{fig:nmse_s30}
    }
    &
    \subfigure[Noisy Case: $\textrm{SNR}$=20dB]{
      \includegraphics[width=0.45\linewidth]{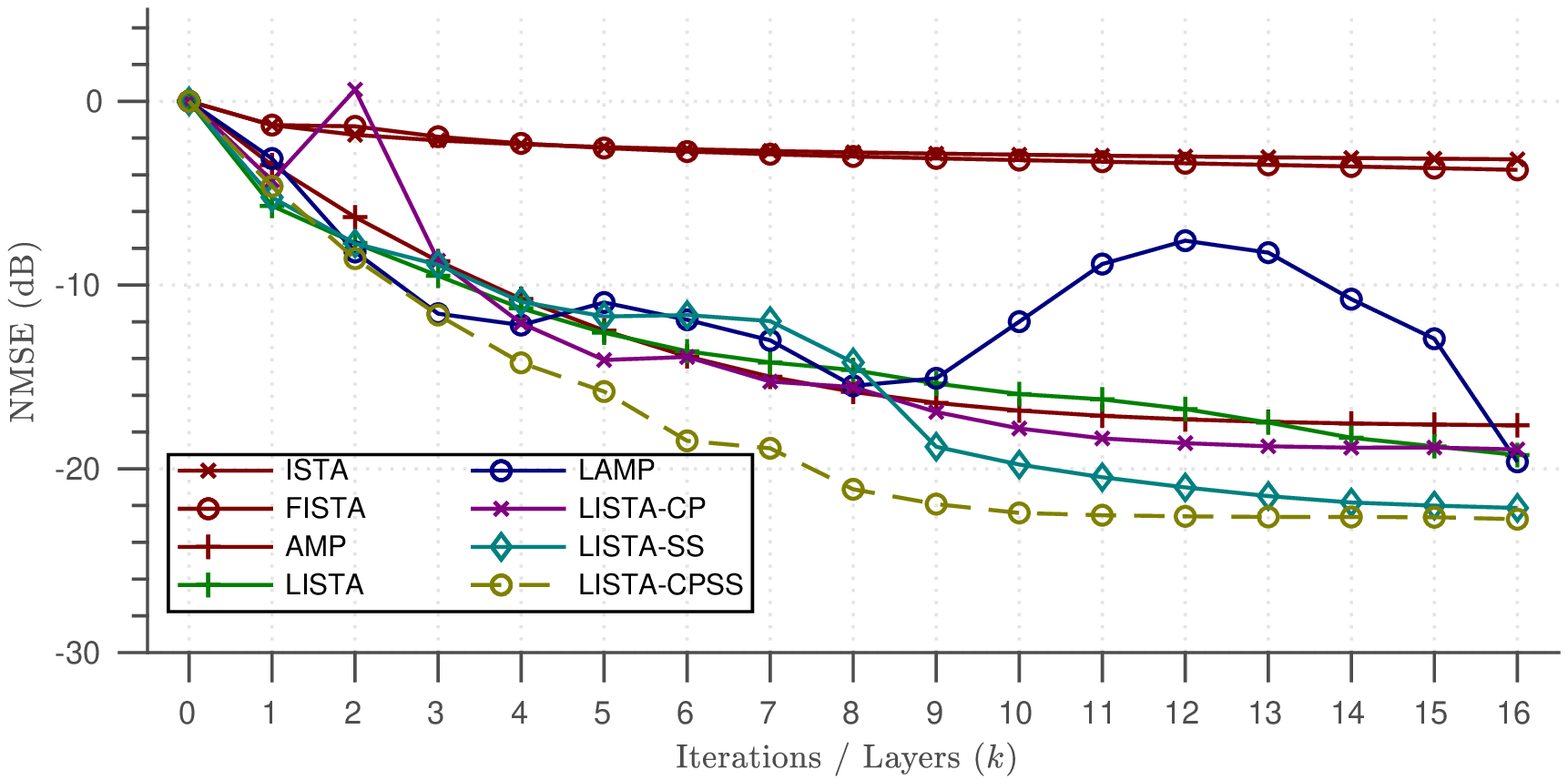}
      \label{fig:nmse_s20}
    }
  \end{tabular}
  \caption{Validation of Theorem \ref{prop:ss}.}
  \vspace{-1em}
  \label{fig:ss}
\end{figure}

\textbf{Performance with Ill-Conditioned Matrix.}
We train LISTA, LAMP, LISTA-CPSS with ill-conditioned matrices $A$ of condition
numbers  $\kappa=5,30,50$. As is shown in Fig. \ref{fig:ill}, as $\kappa$
increases, the performance of LISTA remains stable while LAMP becomes worse, and
eventually inferior to LISTA when $\kappa=50$. Although our LISTA-CPSS also
suffers from ill-conditioning, its performance always stays much better than
LISTA and LAMP.

\begin{figure}[ht]
  \centering
  \begin{tabular}{ccc}
    \hspace{-10mm}
    \subfigure[$\ \kappa=5$]{
      \includegraphics[width=0.34\linewidth]{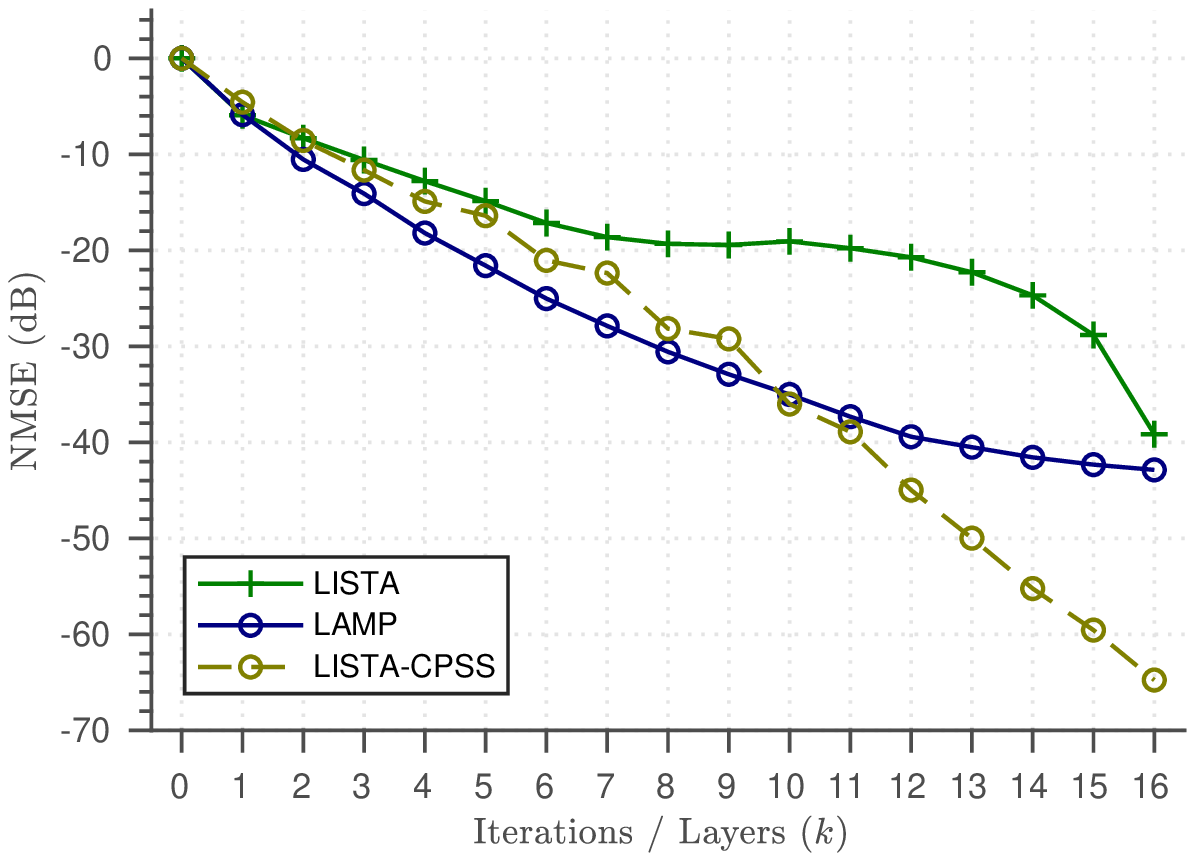}
      \label{fig:nmse_k05}
    }
    &
    \hspace{-4mm}
    \subfigure[$\ \kappa=30$]{
      \includegraphics[width=0.34\linewidth]{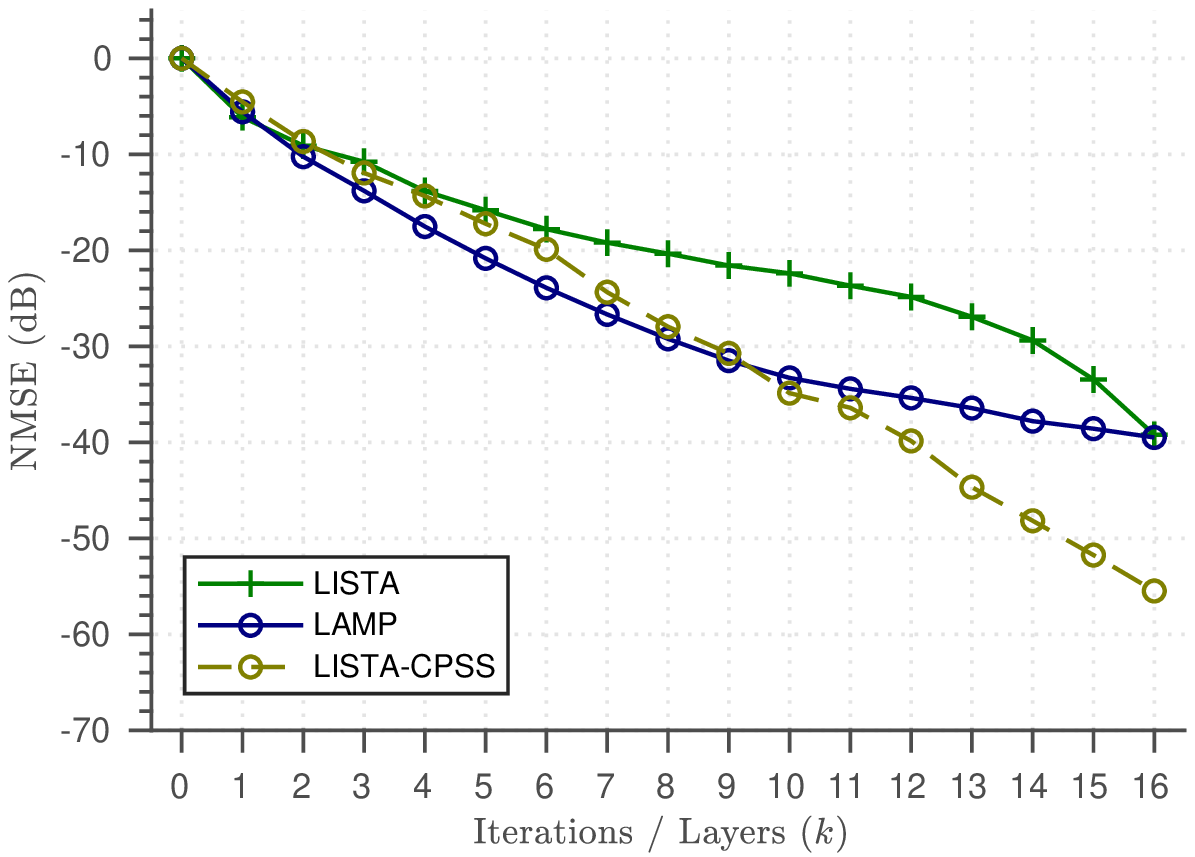}
      \label{fig:nmse_k30}
    }
    &
    \hspace{-4mm}
    \subfigure[$\ \kappa=50$]{
      \includegraphics[width=0.34\linewidth]{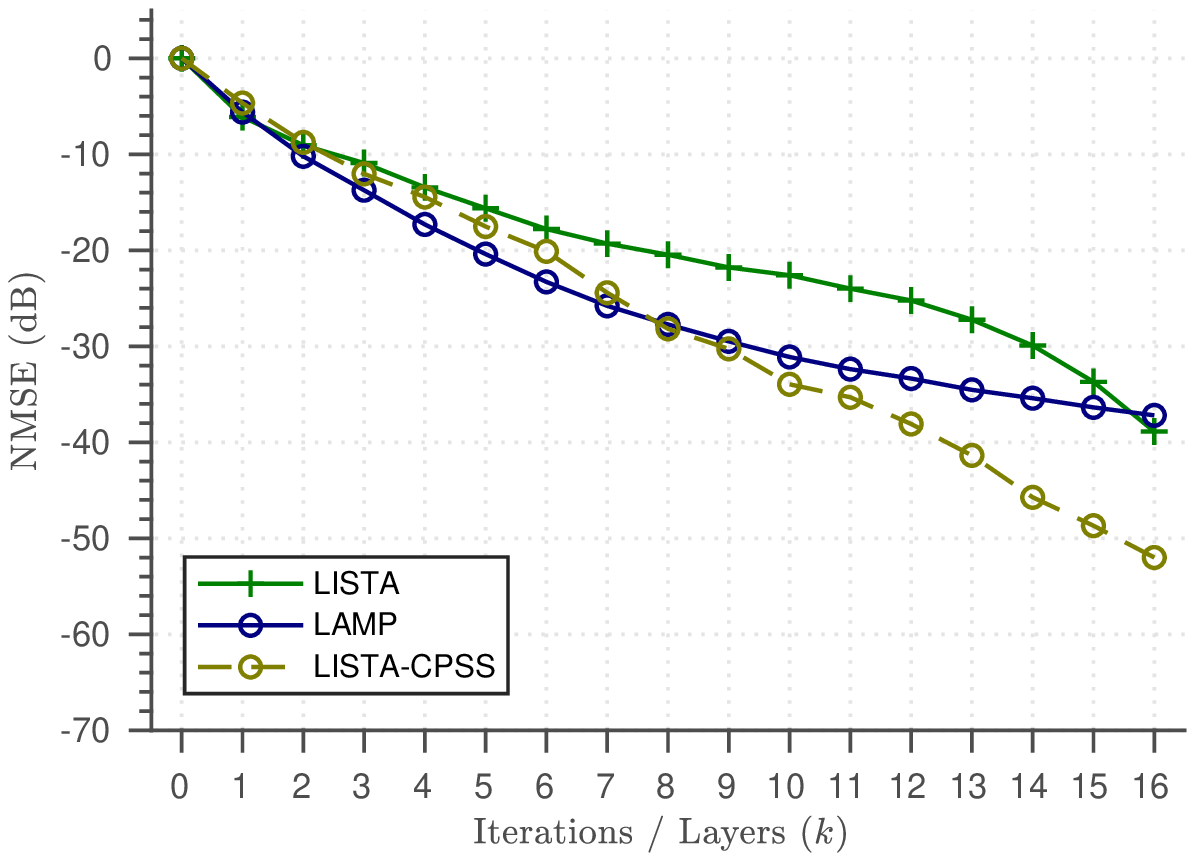}
      \label{fig:nmse_k50}
    }
  \end{tabular}
  \caption{Performance in ill-conditioned situations (SNR = $\infty$).}
  \label{fig:ill}
  \vspace{-1.5em}
\end{figure}

\vspace{-0.5em}
\subsection{Natural Image Compressive Sensing}
\vspace{-0.5em}

\textbf{Experiments Setting.}
We perform a compressive sensing (CS) experiment on natural images (patches). We
divide the BSD500~\cite{MartinFTM01} set into a training set of 400 images, a
validation set of 50 images, and a test set of 50 images. For training, we
extract 10,000 patches $f \in \mathbb{R}^{16\times 16}$ at random positions of
each image, with all means removed. We then learn a dictionary $D \in
\mathbb{R}^{256\times 512}$ from them, using a block proximal gradient method
\cite{xu2013block}. For each testing image, we divide it into non-overlapping
$16\times 16$ patches.  A Gaussian sensing matrices
$\Phi \in \mathbb{R}^{m \times 256}$ is created in the same manner as in Sec.
4.1, where $\frac{m}{256}$ is the CS ratio.

Since $f$ is typically not exactly sparse under the dictionary $D$, Assumptions
\ref{assume:basic} and \ref{assume:basic2} no longer strictly hold. The primary
goal of this experiment is thus to show that our proposed techniques remain
robust and practically useful in non-ideal conditions, rather than beating all
CS state-of-the-arts.

\textbf{Network Extension.}
In the real data case, we have no ground-truth sparse code available as the
regression target for the loss function (\ref{eq:train}). In order to bypass
pre-computing sparse codes $f$ over $D$ on the training set, we are inspired by
\cite{zhou2018sc2net}: first using layer-wise pre-training with a reconstruction
loss w.r.t. dictionary $D$ plus an $l_1$ loss, shown in
(\ref{eq:extended_train}), where $k$ is the layer index and $\Theta^k$ denotes
all parameters in the $k$-th and previous layers; then appending another
learnable fully-connected layer (initialized by $D$) to LISTA-CPSS and perform
an end-to-end training with the cost function (\ref{eq:joint_train}).
\begin{align}
  L^k(\Theta^k)=&\sum_{i=1}^N
  \|f_i - D\cdot x_i^k(\Theta^k)\|_2^2+\lambda\|x_i^k(\Theta^k)\|_1
  \label{eq:extended_train}\\
  L(\Theta, W_D)=&\sum_{i=1}^N
  \|f_i - W_D\cdot x_i^K(\Theta)\|_2^2+\lambda\|x_i^K(\Theta)\|_1
  \label{eq:joint_train}
\end{align}

\textbf{Results.}
The results are reported in Table \ref{tab:cs}.  We build CS models at the
sample rates of $20\%,30\%,40\%,50\%,60\%$ and test on the standard Set 11
images as in \cite{kulkarni2016reconnet}. We compare our results with three
baselines: the classical iterative CS solver, TVAL3 \cite{li2013efficient}; the
``black-box'' deep learning CS solver, Recon-Net \cite{kulkarni2016reconnet};a
$l_0$-based network unfolded from IHT algorithm \cite{blumensath2009iterative},
noted as LIHT; and the baseline LISTA network, in terms of PSNR (dB)
\footnote{ We applied TVAL3, LISTA and LISTA-CPSS on $16 \times 16$ patches to
be fair. For Recon-Net, we used their default setting working on $33 \times 33$
patches, which was  verified to perform better than using smaller patches.}.
We build 16-layer LIHT, LISTA and LISTA-CPSS networks and set $\lambda=0.2$. For
LISTA-CPSS, we set $p\% = 0.4\%$ more entries into the support in each layer for
support selection. We also select support w.r.t. a percentage of the largest
magnitudes within \textit{the whole batch} rather than within a single sample as
we do in theorems and simulated experiments, which we emprically find is
beneficial to the recovery performance. Table \ref{tab:cs} confirms LISTA-CPSS
as the best performer among all. The advantage of LISTA-CPSS and LISTA over
Recon-Net also endorses the incorporation of the unrolled sparse solver
structure into deep networks.
\begin{table}%{0.65\textwidth}
  \caption{The Average PSRN (dB) for Set 11 test images with CS ratio ranging
    from 0.2 to 0.6}
  \label{tab:cs}
  \centering
  \begin{tabular}{cccccc}
    \toprule
    Algorithm  & 20\% & 30\% & 40\% & 50\% & 60\% \\
    \midrule
    TVAL3      & 25.37     & 28.39     & 29.76     & 31.51     & 33.16     \\
    Recon-Net  & 27.18     & 29.11     & 30.49     & 31.39     & 32.44     \\
    LIHT       & 25.83     & 27.83     & 29.93     & 31.73     & 34.00     \\
    LISTA      & 28.17     & 30.43     & 32.75     & 34.26     & 35.99     \\
    LISTA-CPSS & \textbf{28.25} & \textbf{30.54} & \textbf{32.87}
    		   & \textbf{34.60} & \textbf{36.39} \\
    \bottomrule
  \end{tabular}
  \vspace{-1em}
\end{table}

\vspace{-0.5em}
\section{Conclusions}
\vspace{-0.5em}
In this paper, we have introduced a partial weight coupling structure to LISTA,
which reduces the number of trainable parameters but does not hurt the
performance. With this structure, unfolded ISTA can attain a linear convergence
rate. We have further proposed support selection, which improves the convergence
rate both theoretically and empirically. Our theories are endorsed by extensive
simulations and a real-data experiment.
We believe that the methodology in this paper can be extended to analyzing and
enhancing other unfolded iterative algorithms.

\subsubsection*{Acknowledgments}

The work by X. Chen and Z. Wang is supported in part by NSF RI-1755701. The work by J. Liu and W. Yin is supported in part by NSF DMS-1720237 and ONR N0001417121. We would also like to thank all anonymous reviewers for their tremendously useful comments to help improve our work.

\bibliographystyle{unsrt}
\bibliography{nips_2018}

\newpage
\pagenumbering{arabic}
\setcounter{page}{1}
\appendix

\begin{center}
{\Large \centering \textbf{Theoretical Linear Convergence of Unfolded ISTA and Its Practical Weights and Thresholds (Supplementary Material)}} \end{center}

\paragraph{Some notation} For any $n$-dimensional vector $x\in \Re^n$, subscript
$x_S$ means the part of $x$ that is supported on the index set $S$:
\[
  x_S \triangleq [x_{i_1},x_{i_2},\cdots,x_{i_{|S|}}]^T, \quad
  i_1,\cdots,i_{|S|} \in S,\quad i_1 \leq i_2 \leq \cdots \leq i_{|S|},
\]
where $|S|$ is the size of set $S$. For any matrix $W \in \Re^{m \times n}$,
\[
  \begin{aligned}
    W(S,S) \triangleq & \begin{bmatrix}
    W(i_1, i_1), W(i_1, i_2) ,\cdots, W(i_1, i_{|S|})\\
    W(i_2, i_1), W(i_2, i_2) ,\cdots, W(i_2, i_{|S|})\\
    \cdots \\
    W(i_{|S|}, i_1), W(i_{|S|}, i_2) ,\cdots, W(i_{|S|}, i_{|S|})\\
    \end{bmatrix},\quad i_1,\cdots,i_{|S|} \in S,\quad i_1 \leq i_2 \leq \cdots \leq i_{|S|},\\
    W(S,:) \triangleq & \begin{bmatrix}
    W(i_1, 1), W(i_1, 2) ,\cdots, W(i_1, n)\\
    W(i_2, 1), W(i_2, 2) ,\cdots, W(i_2, n)\\
    \cdots \\
    W(i_{|S|}, 1), W(i_{|S|}, 2) ,\cdots, W(i_{|S|}, n)\\
    \end{bmatrix},\quad i_1,\cdots,i_{|S|} \in S,\quad i_1 \leq i_2 \leq \cdots \leq i_{|S|},\\
    W(:,S) \triangleq & \begin{bmatrix}
    W(1, i_1), W(1, i_2) ,\cdots, W(1, i_{|S|})\\
    W(2, i_1), W(2, i_2) ,\cdots, W(2, i_{|S|})\\
    \cdots \\
    W(n, i_1), W(n, i_2) ,\cdots, W(n, i_{|S|})\\
    \end{bmatrix},\quad i_1,\cdots,i_{|S|} \in S,\quad i_1 \leq i_2 \leq \cdots \leq i_{|S|}.
  \end{aligned}
\]

\section{Proof of Theorem~\ref{prop:necessary}}

\begin{proof} By LISTA model (\ref{eq:gen_ista}), the output of the $k$-th layer $x^k$ depends on parameters, observed signal $b$ and initial point $x^0$: $x^k\Big( \{W_1^\tau, W_2^\tau, \theta^\tau \}_{\tau=0}^{k-1},b,x^0 \Big)$. Since we assume $(x^*,\varepsilon)\in \X(B,s,0)$, the noise $\varepsilon=0$. Moreover, $A$ is fixed and $x^0$ is taken as $0$. Thus, $x^k$ therefore depends on parameters and $x^*$: $x^k\Big( \{W_1^\tau, W_2^\tau, \theta^\tau \}_{\tau=0}^{k-1},x^* \Big)$ In this proof, for simplicity, we use $x^k$ denote $x^k\Big( \{W_1^\tau, W_2^\tau, \theta^\tau \}_{\tau=0}^{k-1},x^* \Big)$.

\paragraph{Step 1}Firstly, we prove $\theta^k \to 0$ as $k \to \infty$.

We define a subset of $\X(B,s,0)$ given $0 < \tilde{B} \leq B$:
\[\tilde{\X}(B,\tilde{B}, s,0) \triangleq  \Big\{(x^*,\varepsilon) \Big| \tilde{B} \leq |x^\ast_i | \leq B, \forall i,~ \|x^\ast\|_0 \leq s, \varepsilon=0 \Big\} \subset \X(B,s,0).\]
Since $x^k\to x^*$ uniformly for all $(x^*,0)\in\X(B,s,0)$, so does for all $(x^*,0)\in\tilde{\X}(B,B/10, s,0)$. Then there exists a uniform $K_1>0$ for all $(x^*,0)\in\tilde{\X}(B,B/10, s,0)$, such that $|x^k_i-x^*_i| < B/10$ for all $i = 1,2,\cdots,n$ and $k\geq K_1$, which implies
\begin{equation}
\label{eq:support_match}
\text{sign}(x^k) = \text{sign}(x^*), \quad \forall k\geq K_1.
\end{equation}
The relationship between $x^k$ and $x^{k+1}$ is
\[x^{k+1} = \eta_{\theta^k}\Big( W^k_2x^k + W^k_1b \Big).\]
Let $S = \text{support}(x^*)$. Then, (\ref{eq:support_match}) implies that, for any $k \geq K_1$ and $(x^*,0)\in\tilde{\X}(B,B/10, s,0)$, we have
\[
x^{k+1}_S = \eta_{\theta^k}\Big( W^k_2(S,S) x^k_S + W^k_1(S,:) b\Big).
\]
The fact (\ref{eq:support_match}) means $x^{k+1}_i \neq 0,\forall i \in S$. By the definition $\eta_{\theta}(x) = \text{sign}(x)\max(0,|x|-\theta)$, as long as $\eta_{\theta}(x)_i \neq 0$, we have  $\eta_{\theta}(x)_i = x_i - \theta \text{ sign}(x_i)$. Thus,
\[x^{k+1}_S =  W^k_2(S,S) x^k_S + W^k_1(S,:) b - \theta^k \text{ sign}(x^*_S). \]
Furthermore, the uniform convergence of $x^k$ tells us, for any $\epsilon>0$ and $(x^*,0)\in\tilde{\X}(B,B/10, s,0)$, there exists a large enough constant $K_2>0$ and $\xi_1,\xi_2 \in \Re^{|S|}$ such that $x^k_S = x^*_S + \xi_1,x^{k+1}_S = x^*_S + \xi_2$ and $\|\xi_1\|_2 \leq \epsilon,\|\xi_2\|_2\leq \epsilon$. Then
\[
x^*_S + \xi_2 = W^k_2(S,S) (x^*_S + \xi_1) + W^k_1(S,:) b - \theta^k \text{ sign}(x^*_S) .
\]
Since the noise is supposed to be zero $\varepsilon=0$, $b=Ax^*$. Substituting $b$ with $Ax^*$ in the above equality, we obtain
\[
x^*_S = W^k_2(S,S) x^*_S + W^k_1(S,:) A(:,S) x^*_S - \theta^k  \text{ sign}(x^*_S)+ \xi,
\]where $\|\xi\|_2 = \|W^k_2(S,S)\xi_1 - \xi_2\|_2\leq (1+B_W)\epsilon$, $B_W$ is defined in Theorem \ref{prop:necessary}.
Equivalently,
\begin{equation}
\label{eq:proof_1_2}
\Big(I-W^k_2(S,S)-W^k_1A(S,S)\Big)x^*_S = \theta^k \text{ sign}(x^*_S) - \xi.
\end{equation}
For any $(x^*,0)\in\tilde{\X}(B/2,B/10, s,0)$, $(2x^*,0)\in\tilde{\X}(B,B/10, s,0)$ holds.
Thus, the above argument holds for all $2x^*$ if $(x^*,0)\in\tilde{\X}(B/2,B/10, s,0)$. Substituting $x^*$ with $2x^*$ in (\ref{eq:proof_1_2}), we get
\begin{equation}
\label{eq:proof_1_3}
\Big(I-W^k_2(S,S)-W^k_1A(S,S)\Big)2x^*_S = \theta^k \text{ sign}(2x^*_S) - \xi' = \theta^k\text{ sign}(x^*_S) - \xi',
\end{equation}
where $\|\xi'\|_2 \leq (1+B_W)\epsilon$. Taking the difference between (\ref{eq:proof_1_2}) and (\ref{eq:proof_1_3}), we have
\begin{equation}
\label{eq:proof_1_1}
\Big(I-W^k_2(S,S)-W^k_1A(S,S)\Big)x^*_S = -\xi'+\xi.
\end{equation}
Equations (\ref{eq:proof_1_2}) and (\ref{eq:proof_1_1}) imply \[\theta^k \text{ sign}(x^*_S) - \xi = -\xi'+\xi.\]Then $\theta^k$ can be bounded with
\begin{equation}
\label{eq:proof_theta_bdd}
\theta^k \leq \frac{3(1+B_W)}{\sqrt{|S|}}\epsilon,\quad \forall k\geq \max(K_1,K_2).
\end{equation}
The above conclusion holds for all $|S|\geq 1$. Moreover, as a threshold in $\eta_{\theta}$, $\theta^k\geq 0$. Thus, $0\leq \theta^k \leq 3(1+B_W)\epsilon$ for any $\epsilon>0$ as long as $k$ large enough. In another word, $\theta^k\to 0$ as $k \to \infty$.

\paragraph{Step 2}We prove that $I-W^k_2-W^k_1A \to 0$ as $k \to \infty$.

LISTA model (\ref{eq:gen_ista}) and $b=Ax^*$ gives
\[
\begin{aligned}
x^{k+1}_S =& \eta_{\theta^k}\Big( W^k_2(S,:) x^k + W^k_1(S,:) b\Big)\\
=& \eta_{\theta^k}\Big( W^k_2(S,:) x^k + W^k_1(S,:) A(:,S)x^*_S\Big)\\
\in & W^k_2(S,:) x^k + W^k_1(S,:) A(:,S)x^*_S - \theta^k \partial \ell_1(x^{k+1}_S),
\end{aligned}
\]
where $\partial \ell_1(x)$ is the sub-gradient of $\|x\|_1$. It is a set defined component-wisely:
\begin{equation}
\label{eq:proof_subgradient}
\partial \ell_1(x)_i =\begin{cases}
                        \{\text{sign}(x_i)\}\quad &\text{if $x_i \neq 0$}, \\
                        [-1,1] \quad &\text{if $x_i = 0$}.
                    \end{cases}
\end{equation}
The uniform convergence of $x^k$ implies, for any $\epsilon>0$ and $(x^*,0)\in\X(B, s,0)$, there exists a large enough constant $K_3>0$ and $\xi_1,\xi_2 \in \Re^{n}$ such that $x^k = x^* + \xi_3,x^{k+1} = x^* + \xi_4$ and $\|\xi_3\|_2 \leq \epsilon,\|\xi_4\|_2\leq \epsilon$. Thus,
\[
x^{*}_S + (\xi_4)_S
\in  W^k_2(S,S) x^*_S + W^k_2(S,:) \xi_3 + W^k_1A (S,S) x^*_S - \theta^k \partial \ell_1(x^{k+1}_S)
\]
\[
\Big(I-W^k_2(S,S)-W^k_1A(S,S)\Big) x^{*}_S
\in   W^k_2(S,:) \xi_3 - (\xi_4)_S  - \theta^k \partial \ell_1(x^{k+1}_S)
\]
By the definition (\ref{eq:proof_subgradient}) of $\partial \ell_1$, every element in $\partial \ell_1(x),\forall x\in \Re$ has a magnitude less than or equal to $1$. Thus, for any $\xi \in \ell_1(x^{k+1}_S)$, we have $\|\xi\|_2 \leq \sqrt{|S|}$, which implies
\[
\Big\|\Big(I-W^k_2(S,S)-W^k_1A(S,S)\Big) x^{*}_S \Big\|_2
\leq   \|W^k_2\|_2 \epsilon + \epsilon + \theta^k \sqrt{|S|}.
\]
Combined with (\ref{eq:proof_theta_bdd}), we obtain the following inequality for all $k\geq \max(K_1,K_2,K_3)$:
\[
\Big\|\Big(I-W^k_2(S,S)-W^k_1A(S,S)\Big) x^{*}_S \Big\|_2
\leq   \|W^k_2\|_2 \epsilon + \epsilon + 3(1+B_W)\epsilon = 4(1+B_W)\epsilon.
\]
The above inequality holds for all $(x^*,0)\in\X(B,s,0)$, which implies, for all $k\geq \max(K_1,K_2,K_3)$,
\[
\begin{aligned}
\sigma_{\text{max}}\Big(I-W^k_2(S,S)-W^k_1A(S,S)\Big) =& \sup_{\substack{\text{support}(x^*) = S\\ \|x^*_i \|_2 = B}} \Big\{ \frac{\|(I-W^k_2(S,S)-W^k_1A(S,S))x^*_S\|_2}{B} \Big\}\\
\leq & \sup_{(x^*,0)\in\X(B,s,0)} \Big\{ \frac{\|(I-W^k_2(S,S)-W^k_1A(S,S))x^*_S\|_2}{B} \Big\}\\
\leq& \frac{4(1+B_W)}{B}\epsilon.
\end{aligned}
\] Since $s\geq 2$, $I-W^k_2(S,S)-W^k_1A(S,S) \to 0$ uniformly for all $S$ with $2 \leq |S|\leq s$. Then, $I-W^k_2-W^k_1A \to 0$ as $k \to \infty$.
\end{proof}

\section{Proof of Theorem \ref{prop:no_ss}}
Before proving Theorem \ref{prop:no_ss}, we introduce some definitions and a lemma.
\begin{definition}
Mutual coherence
$\mu $ of $A \in \Re^{m\times n}$ (each column of $A$ is normalized) is defined as:
\begin{equation}
\label{eq:mutual_coher0}
\mu(A) =  \max_{\substack{i\neq j\\1 \leq i,j \leq n} }|(A_i)^\top A_j|,
\end{equation}
where $A_i$ refers to the $i^{\text{th}}$ column of matrix $A$.

Generalized mutual coherence $\tilde{\mu} $ of $A \in \Re^{m\times n}$ (each column of $A$ is normalized) is defined as:
\begin{equation}
\label{eq:mutual_coher}
\tilde{\mu}(A) = \inf_{\substack{W \in \Re^{m\times n}\\(W_i)^TA_i = 1,1 \leq i \leq n} } \bigg\{ \max_{\substack{i\neq j\\1 \leq i,j \leq n} }|(W_i)^\top A_j|\bigg\}.
\end{equation}
\end{definition}
The following lemma tells us the generalized mutual coherence is attached at some $\tilde{W}\in\Re^{m\times n}$.
\begin{lemma}
\label{lemma:mucoher}
There exists a matrix $\widetilde{W}\in \Re^{m\times n}$ that attaches the infimum given in (\ref{eq:mutual_coher}):
\[(\widetilde{W}_i)^TA_i = 1,1 \leq i \leq n,\quad \max_{\substack{i\neq j\\1 \leq i,j \leq n} }|(\widetilde{W}_i)^\top A_j| = \tilde{\mu} \]
\end{lemma}
\begin{proof}Optimization problem given in (\ref{eq:mutual_coher}) is a linear programming because it minimizing a piece-wise linear function with linear constraints. Since each column of $A$ is normalized, there is at least one matrix in the feasible set:
\[A \in \{W\in\Re^{m \times n}: (W_i)^TA_i = 1,1 \leq i \leq n\}.\]
In another word, optimization problem (\ref{eq:mutual_coher}) is feasible. Moreover, by the definition of infimum bound (\ref{eq:mutual_coher}), we have
\[0 \leq \tilde{\mu}(A) \leq \max_{\substack{i\neq j\\1 \leq i,j \leq n} }|(A_i)^\top A_j| = \mu(A).\]Thus, $\tilde{\mu}$ is bounded. According to Corollary 2.3 in \cite{bertsimas1997introduction}, a feasible and bounded linear programming problem has an optimal solution.
\end{proof}

Based on Lemma \ref{lemma:mucoher}, we define a set of ``good'' weights which $W^k$s are chosen from:
\begin{definition}
Given $A\in \Re^{m \times n}$, a weight matrix is ``good'' if it belongs to
\begin{equation}
\label{eq:optimal_w}
\X_W(A) = \argmin_{W \in \Re^{m \times n}}
\bigg\{ \max_{1 \leq i,j \leq n }|W_{i,j}|: (W_i)^TA_i = 1,1 \leq i \leq n, \max_{\substack{i\neq j\\1 \leq i,j \leq n} }|(W_i)^\top A_j| = \tilde{\mu} \bigg\}.
\end{equation}
Let $C_W = \max_{1 \leq i,j \leq n }|W_{i,j}|$, if $W \in \X_W(A)$.
\end{definition}

With definitions (\ref{eq:mutual_coher}) and (\ref{eq:optimal_w}), we propose a choice of parameters:
\begin{equation}
\label{eq:wtheta}
W^k \in \X_W(A), \quad \theta^k =  \sup_{(x^*,\varepsilon)\in \X(B,s,\sigma)}\{\tilde{\mu} \|x^k(x^*,\varepsilon) - x^*\|_1\} + C_W \sigma,
\end{equation}
which are uniform for all $(x^*,\varepsilon)\in \X(B,s,\sigma)$. In the following proof line, we prove that (\ref{eq:wtheta}) leads to the conclusion (\ref{eq:linear_conv}) in Theorem \ref{prop:no_ss}.

\paragraph{Proof of Theorem \ref{prop:no_ss}}
\begin{proof}In this proof, we use the notation $x^k$ to replace $x^k(x^*,\varepsilon)$ for simplicity.

\paragraph{Step 1: no false positives.}Firstly, we take $(x^*,\varepsilon) \in \X(B,s,\sigma)$. Let $S = \text{support}(x^*)$. We want to prove by induction that, as long as (\ref{eq:wtheta}) holds, $x^k_i = 0, \forall i \notin S, \forall k$ (no false positives). When $k=0$, it is satisfied since $x^0=0$. Fixing $k$, and assuming $x^k_i = 0, \forall i\notin S$, we have
\[
\begin{aligned}
x^{k+1}_i =& \eta_{\theta^k}\Big(x^k_i - \sum_{j\in S} (W^k_i)^T (Ax^k - b) \Big)\\
=& \eta_{\theta^k}\Big( - \sum_{j\in S} (W^k_i)^T A_j (x^k_j - x^*_j) + (W^k_i)^T\varepsilon \Big),\quad \forall i \notin S.
\end{aligned}
\]
Since $\theta^k = \tilde{\mu} \sup_{x^*,\varepsilon}\{\|x^k - x^*\|_1\} + C_W \sigma$ and $W^k \in \X_W(A)$,
\[
\theta^k \geq \tilde{\mu}\|x^k - x^*\|_1 + C_W \|\varepsilon\|_1\geq \Big|- \sum_{j\in S} (W^k_i)^T A_j (x^k_j - x^*_j) + (W^k_i)^T \varepsilon \Big|,\forall i \notin S,\]
which implies $x^{k+1}_i=0,\forall i\notin S$ by the definition of $\eta_{\theta^k}$. By induction, we have
\begin{equation}
\label{eq:no_false_pos}
x^k_i = 0, \forall i \notin S, \quad \forall k.
\end{equation}
In another word, threshold rule in (\ref{eq:wtheta}) ensures no false positives\footnote{In practice, if we obtain $\theta^k$ by training, but not (\ref{eq:wtheta}), the learned $\theta^k$ may not guarantee no false positives for all layers. However, the magnitudes on the false positives are actually small compared to those on true positives. Our proof sketch are qualitatively describing the learning-based ISTA.} for all $x^k,k=1,2,\cdots$

\paragraph{Step 2: error bound for one $(x^*,\varepsilon)$.}
Next, let's consider the components on $S$. For all $i \in S$,
\[
\begin{aligned}
x^{k+1}_i = & ~ \eta_{\theta^k}\Big(x^k_i - (W^k_i)^T A_S (x^k_S - x^*_S) + (W^k_i)^T\varepsilon \Big)\\
\in &~ x^k_i - (W^k_i)^T A_S (x^k_S - x^*_S) + (W^k_i)^T\varepsilon - \theta^k \partial \ell_1(x^{k+1}_i),
\end{aligned}
\]
where $\partial \ell_1(x)$ is defined in (\ref{eq:proof_subgradient}).
Since $(W^k_i)^TA_i=1$, we have
\[
\begin{aligned}
x^k_i - (W^k_i)^T A_S (x^k_S - x^*_S) =& x^k_i - \sum_{j\in S,j\neq i}(W^k_i)^T A_j (x^k_j - x^*_j) - (x^k_i-x^*_i)\\
=& x^*_i - \sum_{j\in S,j\neq i}(W^k_i)^T A_j (x^k_j - x^*_j) .
\end{aligned}
\]
Then,
\[x^{k+1}_i - x^*_i \in - \sum_{j\in S,j\neq i}(W^k_i)^T A_j (x^k_j - x^*_j) + (W^k_i)^T\varepsilon - \theta^k \partial \ell_1(x^{k+1}_i),\quad\forall i \in S.\]
By the definition (\ref{eq:proof_subgradient}) of $\partial \ell_1$, every element in $\partial \ell_1(x),\forall x\in \Re$ has a magnitude less than or equal to $1$. Thus, for all $i \in S$,
\[
\begin{aligned}
|x^{k+1}_i - x^*_i| \leq & \sum_{j\in S,j\neq i} \Big|(W^k_i)^T A_j\Big| |x^k_j - x^*_j| + \theta^k + |(W^k_i)^T \varepsilon|\\
\leq & \tilde{\mu} \sum_{j\in S,j\neq i} |x^k_j-x^*_j| + \theta^k + C_W\|\varepsilon\|_1
\end{aligned}
\]
Equation (\ref{eq:no_false_pos}) implies $\|x^{k}-x^*\|_1 = \|x^k_S - x^*_S\|_1$ for all $k$.
Then
\[
\begin{aligned}
\|x^{k+1} - x^*\|_1 = \sum_{i\in S} |x^{k+1}_i - x^*_i|
\leq  & \sum_{i\in S} \Big( \tilde{\mu} \sum_{j\in S,j\neq i} |x^k_j-x^*_j| + \theta^k  + C_W \sigma\Big) \\
 = & \tilde{\mu} (|S|-1) \sum_{i\in S} |x^k_i-x^*_i| + \theta^k |S| + |S|C_W \sigma \\
\leq & \tilde{\mu}(|S|-1) \|x^{k} - x^*\|_1 + \theta^k |S| + |S|C_W \sigma
\end{aligned}
 \]

\paragraph{Step 3: error bound for the whole data set.}
Finally, we take supremum over $(x^*, \varepsilon) \in \X(B,x,\sigma)$, by $|S|\leq s$,
\[\sup_{x^*,\varepsilon} \{\|x^{k+1} - x^*\|_1\} \leq \tilde{\mu}(s-1)\sup_{x^*,\varepsilon}\{ \|x^{k} - x^*\|_1 \} + s \theta^k  + s C_W\sigma.\]
By
$
\theta^k = \sup_{x^*,\varepsilon} \{\tilde{\mu} \|x^k-x^*\|_1 \} + C_W \sigma
$,
we have
\[
\sup_{x^*,\varepsilon} \{\|x^{k+1}-x^*\|_1\} \leq (2\tilde{\mu} s-\tilde{\mu} ) \sup_{x^*,\varepsilon}  \{\|x^k-x^*\|_1\} + 2s C_W \sigma.
\]
By induction, with $c=-\log(2\tilde{\mu} s - \tilde{\mu}), C =  \frac{2s C_W}{1 + \tilde{\mu}-2\tilde{\mu} s}$, we obtain
\[\begin{aligned}
\sup_{x^*,\varepsilon} \{\|x^{k+1}-x^*\|_1\} \leq &  (2\tilde{\mu} s-\tilde{\mu} )^{k+1} \sup_{x^*,\varepsilon}  \{\|x^0 - x^*\|_1\} + 2s C_W \sigma \Big(\sum_{\tau=0}^{k+1} (2\tilde{\mu} s-\tilde{\mu} )^\tau \Big)\\
\leq & (2\tilde{\mu} s-\tilde{\mu} )^k s B + C\sigma  = sB \exp(-ck) + C\sigma.
\end{aligned}
\]
Since $\|x\|_2 \leq \|x\|_1$ for any $x\in \Re^{n}$
%with $\text{support}(x)=S$
, we can get the upper bound for $\ell_2$ norm:
\[\sup_{x^*,\varepsilon} \{\|x^{k+1}-x^*\|_2\} \leq \sup_{x^*,\varepsilon} \{\|x^{k+1}-x^*\|_1\}
\leq s B \exp(-ck) + C\sigma .\]
As long as $s < (1+1/\tilde{\mu})/2$, $c = -\log(2\tilde{\mu} s - \tilde{\mu}) > 0$, then the error bound (\ref{eq:linear_conv}) holds uniformly for all $(x^*,\varepsilon)\in\X(B,s,\sigma)$.
\end{proof}

\section{Proof of Theorem \ref{prop:ss}}

\begin{proof}
In this proof, we use the notation $x^k$ to replace $x^k(x^*,\varepsilon)$ for simplicity.

\paragraph{Step 1: proving (\ref{eq:linear_ss}).}
Firstly, we assume Assumption 1 holds. Take $(x^*,\varepsilon) \in \X(B,s,\sigma)$. Let $S = \text{support}(x^*)$. By the definition of selecting-support operator ${\eta_\mathrm{ss}}_{\theta^k}^{p^k}$, using the same argument with the proof of Theorem \ref{prop:no_ss}, we have LISTA-CPSS also satisfies $x^k_i = 0, \forall i \notin S, \forall k$ (no false positive) with the same parameters as (\ref{eq:wtheta}).

For all $i \in S$, by the definition of ${\eta_\mathrm{ss}}_{\theta^k}^{p^k}$, there exists $\xi^k\in\Re^{n}$ such that
\[
\begin{aligned}
x^{k+1}_i = & {\eta_\mathrm{ss}}_{\theta^k}^{p^k} \Big( x^k_i - (W^k_i)^T A_S (x^k_S - x^*_S) + (W^k_i)^T \varepsilon \Big)\\
=& x^k_i - (W^k_i)^T A_S (x^k_S - x^*_S) + (W^k_i)^T \varepsilon - \theta^k \xi^k_i ,
\end{aligned}
\]
where
\[
\xi^k_i \begin{cases} = 0\quad&\text{if $i \notin S$}\\
                        \in [-1,1]\quad &\text{if $i \in S, x^{k+1}_i= 0$} \\
                         =\text{sign}(x^{k+1}_i)\quad &\text{if $i \in S, x^{k+1}_i \neq 0$, } i\notin S^{p^k}(x^{k+1}),\\
                         = 0 \quad &\text{if $i \in S, x^{k+1}_i \neq 0$, } i\in S^{p^k}(x^{k+1}).
                    \end{cases}
\]
The set $S^{p^k}$ is defined in (\ref{eq:spk}). Let
\[
S^k(x^*,\varepsilon) = \{i|i \in S, x^{k+1}_i \neq 0, i\in S^{p^k}(x^{k+1})\},
\]
where $S^k$ depends on $x^*$ and $\varepsilon$ because $x^{k+1}$ depends on $x^*$ and $\varepsilon$. Then, using the same argument with that of LISTA-CP (Theorem \ref{prop:no_ss}), we have
\[\|x^{k+1}_S - x^*_S\|_1 \leq \tilde{\mu}(|S|-1) \|x^{k}_S - x^*_S\|_1 + \theta^k \big(|S|-|S^k(x^*,\varepsilon)|\big) + |S| C_W \|\varepsilon\|_1. \]
Since $x^k_i = 0, \forall i \notin S$, $\|x^{k}-x^*\|_2 = \|x^k_S - x^*_S\|_2$ for all $k$. Taking supremum over $(x^*, \varepsilon) \in \X(B,s,\sigma)$, we have
\[\sup_{x^*,\varepsilon}\|x^{k+1} - x^*\|_1 \leq (\tilde{\mu} s -1)\sup_{x^*,\varepsilon}\|x^{k} - x^*\|_1 + \theta^k (s-\inf_{x^*,\varepsilon}|S^k(x^*,\varepsilon)|) + s C_W \sigma. \]
By
$
\theta^k = \sup_{x^*,\varepsilon} \{\tilde{\mu} \|x^k-x^*\|_1 \} + C_W \sigma
$,
we have
\[
\sup_{x^*,\varepsilon} \{\|x^{k+1}-x^*\|_1\} \leq \Big(2\tilde{\mu} s - \tilde{\mu} - \tilde{\mu} \inf_{x^*,\varepsilon}|S^k(x^*,\varepsilon)| \Big) \sup_{x^*,\varepsilon}  \{\|x^k-x^*\|_1\} + 2 s C_W \sigma.
\]
Let
\[
\begin{aligned}
c_\mathrm{ss}^k = & -\log \Big( 2\tilde{\mu} s - \tilde{\mu} - \tilde{\mu} \inf_{x^*,\varepsilon}|S^k(x^*,\varepsilon)| \Big)\\
C_{\mathrm{ss}} = & 2 s C_W  \sum_{k=0}^{\infty} \prod_{t=0}^k\exp(-c_\mathrm{ss}^t)).
\end{aligned}
\]
Then, \[
\begin{aligned}
&\sup_{x^*,\varepsilon} \{\|x^{k}-x^*\|_1\}\\
\leq& \Big(\prod_{t=0}^{k-1}\exp(-c_\mathrm{ss}^t)\Big) \sup_{x^*,\varepsilon} \{\|x^{0}-x^*\|_1\} + 2 s C_W   \bigg(\prod_{t=0}^0\exp(-c_\mathrm{ss}^t)) + \cdots + \prod_{t=0}^{k-1}\exp(-c_\mathrm{ss}^t))\bigg)\sigma\\
\leq & s B \Big(\prod_{t=0}^{k-1}\exp(-c_\mathrm{ss}^t)\Big) + C_{\mathrm{ss}}\sigma \leq B \exp\Big(-\sum_{t=0}^{k-1}c_\mathrm{ss}^t\Big) + C_{\mathrm{ss}}\sigma.
\end{aligned}
\]
With $\|x\|_2 \leq \|x\|_1$, we have
\[\sup_{x^*,\varepsilon} \{\|x^{k}-x^*\|_2\} \leq  \sup_{x^*,\varepsilon} \{\|x^{k}-x^*\|_1\} \leq s B \Big(\prod_{t=0}^{k-1}\exp(-c_\mathrm{ss}^t)\Big) + C_{\text{ss}}\sigma.\]
Since $|S^k|$ means the number of elements in $S^k$, $|S^k|\geq 0$. Thus, $c_\mathrm{ss}^k \geq c$ for all $k$.
Consequently,
\[C_{\mathrm{ss}} \leq 2 s C_W   \Big(\sum_{k=0}^{\infty} \exp(-ck)) \Big)  = 2 s C_W  \Big(\sum_{k=0}^{\infty} (2\tilde{\mu} s - \tilde{\mu})^k \Big)  = \frac{2 s C_W}{1+\tilde{\mu}-2\tilde{\mu} s}  = C.\]

\paragraph{Step 2: proving (\ref{eq:linear_ss2}).}
Secondly, we assume Assumption 2 holds. Take $(x^*,\varepsilon) \in \bar{\X}(B,\underline{B}, s,\sigma) $.  The parameters are taken as
\[
W^k \in \X_W(A), \quad \theta^k =  \sup_{(x^*,\varepsilon)\in \bar{\X}(B,\underline{B}, s,\sigma)}\{\tilde{\mu} \|x^k(x^*,\varepsilon) - x^*\|_1\} + C_W \sigma.
\]
With the same argument as before, we get
\[\sup_{(x^*,\varepsilon)\in \bar{\X}(B,\underline{B}, s,\sigma)} \{\|x^{k}-x^*\|_2\}
\leq s B \exp\Big(-\sum_{t=0}^{k-1}\tilde{c}_\mathrm{ss}^t\Big) + \tilde{C}_{\text{ss}}\sigma,\]
where \[
\begin{aligned}
\tilde{c}_\mathrm{ss}^k = & -\log \Big( 2\tilde{\mu} s - \tilde{\mu} - \tilde{\mu} \inf_{(x^*,\varepsilon)\in \bar{\X}(B,\underline{B}, s,\sigma)} |S^k(x^*,\varepsilon)| \Big) \geq c \\
\tilde{C}_{\mathrm{ss}} = & 2 s C_W   \bigg(\sum_{k=0}^{\infty} \prod_{t=0}^k\exp(-\tilde{c}_\mathrm{ss}^t)) \bigg) \leq C.
\end{aligned}
\]
Now we consider $S^k$ in a more precise way. The definition of $S^k$ implies
\begin{equation}
\label{eq:support_eval}
|S^k(x^*,\varepsilon)| = \min \big(p^k,\# \text{ of non-zero elements of }x^{k+1}\big).
\end{equation}
By Assumption \ref{assume:basic2}, it holds that $\|x^*\|_1 \geq \underline{B} \geq 2 C\sigma$. Consequently, if $k > 1/c (\log(s B/C\sigma))$, then
\[s B\exp(-ck) + C\sigma < 2C \sigma \leq \|x^*\|_1,\]
which implies
\[\|x^{k+1}-x^*\|_1 \leq s B (\prod_{t=0}^{k}\exp(-\tilde{c}_\mathrm{ss}^t)) + \tilde{C}_{\mathrm{ss}} \sigma \leq s B\exp(-ck) + C\sigma < \|x^*\|_1.\]
Then $\# \text{ of non-zero elements of }x^{k+1} \geq 1$. (Otherwise, $\|x^{k+1}-x^*\|_1 = \|0-x^*\|_1$, which contradicts.) Moreover, $p^k= \min(pk, s)$ for some constant $p>0$. Thus, as long as $k\geq 1/p$, we have $p^k\geq 1$. By (\ref{eq:support_eval}), we obtain
\[|S^k(x^*,\varepsilon)| > 0,\quad  \forall k > \max\Big( \frac{1}{p},  \frac{1}{c} \log\Big( \frac{s B}{C\sigma}  \Big) \Big),~\forall (x^*,\varepsilon)\in \bar{\X}(B,\underline{B}, s,\sigma). \]
Then, we have $\tilde{c}_\mathrm{ss}^k > c$ for large enough $k$, consequently, $\tilde{C}_{\mathrm{ss}} < C$.
\end{proof}

\section{The adaptive threshold rule used to produce Fig. \ref{fig:nmse_ista_lista}}

\begin{algorithm2e}[h]
\SetKwInOut{input}{Input}
\SetKwInOut{initial}{Initialization}
\input{Maximum iteration $K$, initial $\lambda^0, \epsilon^0$.}
\initial{Let $x^0=0, \lambda^1 = \lambda^0, \epsilon^1 = \epsilon^0$.}
\For{$k = 1,2,\cdots,K$} {
Conduct ISTA: $x^{k} = \eta_{\lambda^k/L}\Big( x^{k-1} - \frac{1}{L}A^T(Ax^{k-1} - b)\Big)$.\\
\uIf{$\|x^k-x^{k-1}\|<\epsilon^k$}{
Let $\lambda^{k+1} \leftarrow 0.5 \lambda^{k}$, $\epsilon^{k+1} \leftarrow 0.5 \epsilon^{k}$.
}\Else {
Let $\lambda^{k+1} \leftarrow \lambda^{k}$, $\epsilon^{k+1} \leftarrow  \epsilon^{k}$.
}
}
\KwOut{$x^K$}
  \caption{A thresholding rule for LASSO (Similar to that in
    \protect\cite{HaleYinZhang2008_sparse})}
  \label{algo:adaptive}
\end{algorithm2e}

We take $\lambda^0 = 0.2, \epsilon^0 = 0.05$ in our experiments.

\section{Training Strategy}
In this section we have a detailed discussion on the stage-wise training
strategy in empirical experiments. Denote $\Theta =
\{(W^k_1,W^k_2,\theta^k)\}_{k=0}^{K-1}$ as all the weights in the network. Note
that $(W^k_1,W^k_2)$ can be coupled as in (\ref{eq:couple_way}). Denote
$\Theta^\tau = \{(W^k_1,W^k_2,\theta^k)\}_{k=0}^{\tau}$ all the weights in the
$\tau$-th and all the previous layers. We assign a learning multiplier
$c(\cdot)$, which is initialized as 1, to each weight in the network. Define an
initial learning rate $\alpha_0$ and two decayed learning rates
$\alpha_1, \alpha_2$. In real training, we have
$\alpha_1=0.2\alpha_0, \alpha_2=0.02\alpha_0$. Our training strategy is
described as below:
\begin{itemize}
  \item Train the network layer by layer. Training in each layer consists of 3
    stages.
  \item In layer $\tau$, $\Theta^{\tau-1}$ is pre-trained. Initialize
        $c(W^\tau_1),c(W^\tau_2),c(\theta^\tau)=1$. The actual learning rates of
        all weights in the following are multiplied by their learning multipliers.
      \begin{itemize}
          \item Train $(W^\tau_1,W^\tau_2,\theta^\tau)$ the initial learning
            rate $\alpha_0$.
          \item Train $\Theta^\tau = \Theta^{\tau-1} \cup (W^\tau_1,W^\tau_2,\theta^\tau)$
            with the learning rates $\alpha_1$ and $\alpha_2$.
      \end{itemize}
  \item Multiply a decaying rate $\gamma$ (set to 0.3 in experiments) to each
    weight in $\Theta^\tau$.
  \item Proceed training to the next layer.
\end{itemize}

The layer-wise training is widely adopted in previous LISTA-type networks.  We
add the learning rate decaying that is able to stabilize the training process.
It will make the previous layers change very slowly when the training proceeds
to deeper layers because learning rates of first several layers will
exponentially decay and quickly go to near zero when the training process
progresses to deeper layers, which can prevent them varying too far from
pre-trained positions. It works well especially when the unfolding goes deep to
$K> 10$. All models trained and reported in experiments section are trained
using the above strategy.

\textbf{Remark} While adopting the above stage-wise training strategy,
we first finish a complete training pass, calculate the intermediate results and
final outputs, and then draw curves and evaluate the performance based on these
results, instead of logging how the best performance changes when the training
process goes deeper. This manner possibly accounts for the reason why some
curves plotted in Section \ref{sec:simulation} display some unexpected
fluctuations.

\end{document}